\documentclass{siamart220329}

\usepackage[utf8]{inputenc}
\usepackage{amsmath}
\usepackage{amsfonts}
\usepackage{amssymb}
\usepackage{bbm}
\newtheorem{remark}{Remark}

\newtheorem{assumption}{Assumption}

\usepackage{hyperref}
\usepackage{subcaption}
\usepackage{placeins} %

\usepackage{graphicx}
\usepackage{caption}

\usepackage{cleveref}
\usepackage{algorithm}
\usepackage{algorithmicx}
\usepackage{xcolor}

\def\letters{a,b,c,d,e,f,g,h,i,j,k,l,m,n,o,p,q,r,s,t,u,v,w,x,y,z,%
	A,B,C,D,E,F,G,H,I,J,K,L,M,N,O,P,Q,R,S,T,U,V,W,X,Y,Z}
\makeatletter
\@for \@l:=\letters \do{%
	\expandafter\edef\csname bb\@l\endcsname{\noexpand\ensuremath{\noexpand\mathbb{\@l}}}%
	\expandafter\edef\csname bf\@l\endcsname{{\noexpand\bf \@l}}%
	\expandafter\edef\csname cal\@l\endcsname{\noexpand\ensuremath{\noexpand\mathcal{\@l}}}%
	\expandafter\edef\csname eu\@l\endcsname{\noexpand\ensuremath{\noexpand\EuScript{\@l}}}%
	\expandafter\edef\csname frak\@l\endcsname{\noexpand\ensuremath{\noexpand\mathfrak{\@l}}}%
	\expandafter\edef\csname rm\@l\endcsname{{\noexpand\rm \@l}}%
	\expandafter\edef\csname scr\@l\endcsname{\noexpand\ensuremath{\noexpand\mathscr{\@l}}}%
}

\newcommand{\diag}{\operatorname{diag}}
\newcommand{\dd}{\operatorname{d}\!}

\newcommand{\isdef}{\mathrel{\mathrel{\mathop:}=}}
\newcommand{\defis}{\mathrel{=\mathrel{\mathop:}}}

\title{Quantifying uncertainty in spectral clusterings: expectations for perturbed and incomplete data}
\author{J.~D\"olz\thanks{University of Bonn, Institute for Numerical Simulation, Friedrich-Hirzebruch-Allee 7, 53115 Bonn, Germany (\email{doelz@ins.uni-bonn.de}, \email{weygandt@.uni-bonn.de}).} \and J.~Weygandt\footnotemark[1]}

\begin{document}
	\vspace*{-2eM}
	\maketitle
	
	\begin{abstract}
		Spectral clustering is a popular unsupervised learning technique which is able to partition unlabelled data into disjoint clusters of distinct shapes. However, the data under consideration are often experimental data, implying that the data is subject to measurement errors and measurements may even be lost or invalid. These uncertainties in the corrupted input data induce corresponding uncertainties in the resulting clusters, and the clusterings thus become unreliable.
		
		Modelling the uncertainties as random processes, we discuss a mathematical framework based on random set theory for the computational Monte Carlo approximation of statistically expected clusterings in case of corrupted, i.e., perturbed, incomplete, and possibly even additional, data. We propose several computationally accessible quantities of interest and analyze their consistency in the infinite data point and infinite Monte Carlo sample limit. Numerical experiments are provided to illustrate and compare the proposed quantities.
	\end{abstract}

	\section{Introduction}
	\subsection{Motivation}
	Spectral clustering \cite{Von2007} is a popular unsupervised learning technique that partitions unlabelled data points into disjoint clusters. Although many questions regarding its theoretical properties remain open, spectral clustering has gained in popularity due to its capability to recognize clusters of distinct shapes, such as entangled half circles or a point cloud within a circle, when compared to other clustering algorithms \cite{KVV2004,NJW2001,Von2007}.
	
	Unfortunately, the data under consideration are often experimental, implying that the data points on which the clustering algorithm operates are often corrupted. That is, they are affected by measurement errors. Moreover, measurements may even be lost or invalid due to errors in the data collection process. This affects the outcome of the clustering algorithm, and it is a natural question whether a given outcome is (statistically) representative.
	
	\subsection{Eulerian versus Lagrangian perspective}
	To answer whether the outcome of a clustering is representative we need to understand how clusterings on different data sets can be compared. Modelling the perturbed and incomplete data as a corrupted version of a ``true'' reference data set, essentially two approaches can be distinguished: the Eulerian perspective and the Lagrangian perspective. Within the Eulerian perspective, the corrupted data is considered within the coordinate system of the reference data set, whereas in the Lagrangian perspective, the corrupted data is considered within a mapped version of the coordinate system of the reference data set.
	
	The Lagrangian perspective is natural if a natural one-to-one correspondence between the reference data set and the corrupted data set can be established. However, establishing such a one-to-one correspondence is not always straightforward, in particular when the cardinalities of reference and corrupted data sets differ. Working in the Eulerian approach is less intuitive at first, as it requires a concept of out-of-sample evaluation to compare values of data sets. However, if made feasible, it provides a natural concept to compare data sets of different cardinality. Other than many articles in the literature, which attain a Lagrangian perspective, the approach proposed in this article attains an Eulerian perspective. This also facilitates its consistency analysis in the infinite data point limit.

	\subsection{Related work}
  Since the original introduction of spectral clustering in \cite{DH1973,Fie1973} there has been an extensive amount of research on various variants of spectral clustering, see e.g.~\cite{Von2007} for an in-depth discussion. As spectral clustering is based on the eigenpairs of the graph Laplacian on a similarity graph between the data points, one of the canonical types of analysis is to consider perturbations in the graph Laplacian matrix. Such perturbation-based perspectives are naturally linked to the Lagrangian perspective. A-priori bounds for the stability of spectral clusterings based on spectral gaps have been derived in several articles using matrix perturbation theory \cite{DH1973,NJW2001,PY2020}. A refined analysis based on matrix nearness problems was considered in \cite{AEGL2021,GS2024}. Empirical studies were carried out in \cite{GM1998,HYTJ2008}. In terms of uncertainty quantification and statistical quantities of interest, a Bayesian perturbation model was developed in \cite{DMD2023} and fuzzy spectral clusterings were considered in \cite{RW2013}. An applied problem considering trajectory clustering in ocean dynamics was studied in \cite{VRA2020}. Working within the Lagrangian perspective, the commonality of these approaches is that the cardinality of data sets is required to remain fixed, and incomplete or additional data are difficult to incorporate.
  
  To the best of our knowledge, the only work considering uncertainties in the Eulerian perspective for spectral clustering is \cite{BLSZ2018}. This article considers a semi-supervised setting, where the clustering is known for a subset of the data points and the transfer uncertainty to other data points is to be quantified. Our article applies to the unsupervised setting, i.e., no reference data is available.
  
  We conclude our literature review by noting that the uncertainty quantification of eigenvalue problems itself has gained increasing interest recently \cite{AS2012,CL2024,DE2024,DESZ2024,GGK+2019a,GSH2023}, and that the concepts there naturally transfer to UQ approaches of spectral clustering in the Lagrangian perspective.

	\subsection{Contributions}
	In this article we develop and analyze algorithms to compute expectations of spectral clusterings arising from corrupted data. The corruption can be due to data perturbation or due to missing data, for which we adopt an Eulerian perspective. Specifically:
	\begin{enumerate}
		\item We discuss how the framework of the consistency analysis for spectral clustering from \cite{VBB2008} can be used to adopt an Eulerian perspective to spectral clustering for corrupted data.
		\item We consider the arising clusterings as set-valued random variables, and consider several notions of expectations and their Monte Carlo approximation. We propose a new notion for the specific case of spectral clustering.
		\item We provide a consistency analysis for Gaussian similarity measures for the infinite data and infinite Monte Carlo sample limits, showing that the order in which the limit is taken is irrelevant.
	\end{enumerate}
	These contributions are complemented by numerical examples.
	
	\subsection{Outline}
	The rest of the article is organized as follows. In \cref{sec:preliminaries} we recall the basics of spectral clustering, a continuous analogon, and discuss how spectral clusterings on corrupted data sets can be related to each other. \Cref{sec:expectations} discusses details on the modelling of corrupted data, and relates spectral clusterings on corrupted data to several notions of expectations from random set theory. We further propose a new expectation, which is simple to compute and specific to spectral clustering.  In \cref{sec:mc} we comment on Monte Carlo methods for the approximation of these expectations, whereas the consistency of the estimators in the infinite data point and Monte Carlo sample limit is analyzed in \cref{sec:consistency}. Numerical experiments are provided in \cref{sec:experiments}, before we draw our conclusions in \cref{sec:concl}.
	
	\section{Spectral clustering on corrupted data sets}\label{sec:preliminaries}
	\subsection{Spectral clustering}\label{sec:spectralclustering}
  We recall the basics of (normalized) spectral clustering along the lines of \cite{VBB2008} and refer to \cite{NJW2001,Von2007} for a more detailed introduction. We assume that we are given a finite set of samples $X\subset\calX$ with the data space $\calX$ being a metric space.
	Moreover, we assume to have a notion of similarity to our disposal given through a symmetric, positive semidefinite, continuous \emph{similarity function} $k\colon \calX\times \calX\to\bbR$.
	
	Defining the \emph{similarity matrix} $\bfK_X$ and the \emph{degree matrix} $\bfD_X$ through
	\[
	\bfK_X=[k(x,y)]_{x,y\in X},
	\qquad
	\bfD_X=\diag_{x\in X}\bigg(\sum_{y\in X}k(x,y)\bigg),
	\]
	allows to introduce the (symmetric) normalized \emph{graph Laplacian}
	\[
	\bfL_X=\bfI_{|X|}-\bfD_X^{-1/2}\bfK_X\bfD_X^{-1/2}.
	\]
	It is well known that $\bfL_X$ has non-negative, real eigenvalues $0=\lambda_1\leq \lambda_2\leq\ldots\leq\lambda_{|X|}$. With the graph Laplacian available, a simplistic version of spectral clustering reads as follows:
	\begin{algorithm}[H]
		\caption{Spectral bi-clustering of a data set $X$.}
		\label{alg:spectralclustering}
		\begin{algorithmic}
			\State Given a set of samples $X\subset\calX$, $|X|<\infty$:
			\State Compute the graph Laplacian $\bfL_X$.
			\State Compute the eigenvector $\bfv=[v_x]_{x\in X}$ of the second smallest eigenvalue of $\bfL_X$.
			\State Define clusters $A=\{x\in X\colon v_x\geq0\}$ and $\overline{A}=\{x\in X\colon v_x<0\}$.
		\end{algorithmic}
	\end{algorithm}
	For an illustrative explanation why this approach is reasonable and for extensions to more than two clusters we refer to \cite{Von2007}. Although many of the following elaborations can be extended to more than two clusters, we prefer to stick with the bi-clustering case to simplify exposition.
	
	\subsection{Continuous spectral partitioning}\label{eq:continuouspartitioning}
	One of the challenges for the Eulerian approach we aim for, and for spectral clustering in general, is that it is rather difficult to compare the clustering process for different numbers of samples. To overcome this restriction we will need out of sample extensions of the clusterings. To this end, we follow the approach of \cite{VBB2008} to define the \emph{sample measure}
	\[
	\bbP_X=\frac{1}{|X|}\sum_{x\in X}\delta_{x},
	\]
	the \emph{degree functions}
	\[
	d_X(x)=\int_\calX k(x,y)\dd\bbP_X(y),
	\]
	the \emph{normalized similarity functions}
	\[
	h_X(x,y)=\frac{k(x,y)}{\sqrt{d_X(x)d_X(y)}},\]
	and the (compact) integral operators
	\[
	\calT_X\colon \calC(\calX)\to \calC(\calX),\qquad \calT_Xf(x)=\int_\calX h_X(x,y)f(y)\dd\bbP_X(y),
	\]
	to define
	\begin{align}\label{eq:discreteIO}
		\calU_X&=I-\calT_X.
	\end{align}
	It is straightforward to check that the graph Laplacian and $\calU_X$ are related due to
	\[
	\bfL_X\circ\rho_X=\rho_X\circ \calU_X,
	\]
	where
	\[
	\rho_X\colon \calC(\calX)\to\bbR^{|X|},\qquad
	f\mapsto[f(x)]_{x\in X},
	\]
	is the \emph{evaluation operator}. The following connections between the eigenvectors of $\bfL_X$ and the eigenfunctions of $\calU_X$ can be stated.
	\begin{theorem}[{\cite[Proposition 9]{VBB2008}}]\label{thm:VBB2008prop9}
		\begin{enumerate}
			\item If $(f_X,\lambda_X)\in \calC(\calX)\times\bbR$ is an eigenpair of $\calU_X$, then $(\bfv_X=\rho_Xf_X,\lambda)\in\bbR^{|X|}\times\bbR$ is an eigenpair of $\bfL_X$.
			\item Let $(f_X,\lambda_X)\in \calC(\calX)\times\bbR$ be an eigenpair of $\calU_X$ with $\lambda_X\neq 1$  and $[v_y]_{y\in X}=\rho_Xf_X\in\bbR^{|X|}$. Then $f_X$ has the representation
			\begin{equation}\label{eq:VBB2008eigrep}
				f_X(x)=\frac{1}{1-\lambda}\frac{1}{|X|}\sum_{y\in X}k(x,y)v_y,\qquad x\in\calX.
			\end{equation}
			\item If $(\bfv_X=\rho_Xf_X,\lambda)\in\bbR^{|X|}\times\bbR$ is an eigenpair of $\bfL_X$ with eigenvalue $\lambda_X\neq 1$, then $f_X$ as given in \cref{eq:VBB2008eigrep} is an eigenfunction of $\calU_X$ with eigenvalue $\lambda_X$.
		\end{enumerate}
	\end{theorem}
	The special role of $1$ as an eigenvalue becomes apparent upon noting that $(f_X,1)$ is an eigenpair of $\calU_X$ if and only if $(f,0)$ is an eigenpair of the finite rank operator $\calT_X$.
	
	The significance of the theorem is that it allows to derive a sample dependent, i.e., $X$-dependent, bi-partitioning of $\calX$ as outlined in \cref{alg:spectralbipartitioning}.
	\begin{algorithm}[H]
		\caption{Spectral data-dependent bi-partitioning of data space $\calX$.}
		\label{alg:spectralbipartitioning}
		\begin{algorithmic}
			\State Given a set of samples $X\subset\calX$, $|X|<\infty$:
			\State Compute the graph Laplacian $\bfL_X$.
			\State Compute the eigenvector $\bfv=[v_x]_{x\in X}$ of the second smallest eigenvalue of $\bfL_X$.
			\State Define $f\colon\calX\to\bbR$ as in \cref{eq:VBB2008eigrep}.
			\State Define partitions $B_X=\{x\in\calX\colon f(x)\geq0\}$ and $\overline{B_X}=\{x\in \calX\colon f(x)<0\}$.
		\end{algorithmic}
	\end{algorithm}
	
	\subsection{Relating spectral clusterings on the same data set}\label{sec:samedataset}
	In the following, given two distinct clusterings $A_1,\overline{A_1}$ and $A_2,\overline{A_2}$ of a data set $X$, we will need an automated procedure to find a one-to-one correspondence between the two clusterings. That is, we need to determine whether $A_1$ contains approximately the same data points as $A_2$ or whether it contains approximately the same data points as $\overline{A_2}$ and similarly for $\overline{A_1}$. To this end, it is a fortunate coincidence that the clusterings are uniquely determined by two eigenvectors $\bfv_1,\bfv_2$ of two distinct graph Laplacians which are both elements of $\bbR^{|X|}$. A one-to-one correspondence of the clusterings is then given by the directions in which $\bfv_1$ and $\bfv_2$ are pointing. If they are approximately pointing in the same direction then $A_1$ corresponds to $A_2$. If they are pointing in approximate opposite directions, then $A_1$ corresponds to $\overline{A_2}$. Mathematically, this can be expressed in terms of the scalar product in $\bbR^{|X|}$ as
	\[
	\langle\bfv_1,\bfv_2\rangle_{\bbR^{|X|}}
	\begin{cases}
		\geq 0&\text{relate $A_1$ to $A_2$ and $\overline{A_1}$ to $\overline{A_2}$},\\
		<0&\text{relate $A_1$ to $\overline{A_2}$ and $\overline{A_1}$ to $A_2$}.
	\end{cases}
	\]
	We note that in practice $\bfv_1$ and $\bfv_2$ are usually normalized to the Euclidean unit length by the used eigensolver (if not, this can be done a-posteriori without changing the result of the clustering), such that $\langle\bfv_1,\bfv_2\rangle_{\bbR^{|X|}}$ will only be close to zero if $\bfv_1$ and $\bfv_2$ are close to orthogonal. In this case the two clusterings are hardly comparable. Thus, in practice, we may issue a warning if $|\langle\bfv_1,\bfv_2\rangle_{\bbR^{|X|}}|$ is below a given tolerance. A procedure to align multiple eigenvectors has been discussed in \cite[Section 3.3]{DE2024}.
	
	We close the subsection by remarking that the coefficients $a_n$ in \cref{thm:VBB2008thm15} can be chosen as $a_n=\operatorname{sign}(\langle\bfv_1,\bfv_2\rangle_{\bbR^{|X|}})$.

	\subsection{Relating spectral clusterings on different data sets}
	\Cref{eq:VBB2008eigrep} yields an immediate continuation of the eigenvectors of the graph Laplacian from the samples onto the data space $\calX$ and \cref{alg:spectralbipartitioning} exploits this fact by considering this continuation as a level set function to derive a bi-partitioning of the data space $\calX$. The advantage of the bi-partitioning approach of $\calX$ compared to the bi-clustering of $X$ is that a comparison of level-set functions and level sets is mathematically feasible. That is, given two data sets of finite cardinality $X_1,X_2\subset\calX$ we thus need to relate
	\[
	f_{X_1}
	\quad\text{and}\quad
	f_{X_2}
	\qquad
	\text{given through}~\cref{eq:VBB2008eigrep},
	\]
	or
	\[
	\{B_{X_1},\overline{B_{X_1}}\}
	\quad\text{and}\quad
	\{B_{X_2},\overline{B_{X_2}}\}
	\qquad
	\text{given through}~\cref{alg:spectralbipartitioning}.
	\]
	As the comparison of these quantities is computationally infeasible on data spaces $\calX$ of infinite cardinality, we need to restrict our comparison to a discretized version. The data sets used on which the clustering is performed are immediate candidates for discretization. To this end, with a slight abuse of notation, we introduce the clustering of $X_1$ subject to the data set $X_2$ as
	\begin{equation}\label{eq:biclusteringX1X2}
		\begin{aligned}
			A_{X_1|X_2}&=\rho_{X_1}B_{X_2}=\{x\in X_1\colon (\rho_{X_1}f_{X_2})_x\geq 0\},\\
			\overline{A_{X_1|X_2}}&=\rho_{X_1}\overline{B_{X_2}}=\{x\in X_1\colon (\rho_{X_1}f_{X_2})_x<0\}.
		\end{aligned}
	\end{equation}
	\Cref{thm:VBB2008prop9} immediately implies that $A_{X_1|X_1}=A_{X_1}$ and $\overline{A_{X_1|X_1}}=\overline{A_{X_1}}$. Moreover, on a computational level, the vector $(\rho_{X_1}f_{X_2})_{x\in X_1}$, in \cref{eq:biclusteringX1X2} can be written as a vector
	\begin{equation}\label{eq:efevalX1}
		[(\rho_{X_1}f_{X_2})_x]_{x\in X_1}
		=
		\bfv_{X_1}(X_2)
		=
		\frac{1}{1-\lambda}
		\frac{1}{|X_1|}
		\bfM_{X_1,X_2}\bfv_{X_2}
		=
		\rho_{X_1}f_{X_2},
	\end{equation}
	with a possibly rectangular matrix
	\[
	\bfM_{X_1,X_2}
	=
	[k(x,y)]_{x\in X_1,y\in X_2}.
	\]
	As such, the bi-clustering \cref{eq:biclusteringX1X2} can be computed in closed form.
	
	\section{Expectations for spectral clustering}\label{sec:expectations}
	
	\subsection{Modelling of uncertainties}
	In practice, it may happen that the samples of a data set $X\subset\calX$ themselves are subject to some measurement errors and thus need to be considered as a random deviation of some true value. Moreover, it may happen that the data set suffers from wrong or missing samples. In the following we will model these variations and missing data by considering random perturbations due to additive noise. To this end, let $(\Omega,\Sigma,\bbP_\Omega)$ be a probability space, $\widetilde{X}(\omega)\subset X$ a random subset and
	\begin{align*}
		\pi_1(\omega)=\{\varepsilon(x,\omega)\colon x\in\widetilde{X}(\omega)\}.
	\end{align*}
	The distribution of the measurement noise $\varepsilon\colon X\times\Omega\to X$ and the subset may be modelled according to statistical errors during the data collection process. We further assume that there is a set $\pi_2(\omega)\subset\calX\setminus\pi_1(\omega)$, of random cardinality $|\pi_2(\omega)|<\infty$, of additional data points whose elements follow again a probability measure according to the data collection process. For notational convenience we write
	\[
	\pi(\omega)=\pi_1(\omega)\cup\pi_2(\omega),
	\]
	see also \cref{fig:pi} for an illustration.
	\begin{figure}
		\centering
		\begin{minipage}{.4\linewidth}
			\centering
			reference data set
			
			\smallskip
			\includegraphics[width=\linewidth,clip=true,trim=280 80 280 80]{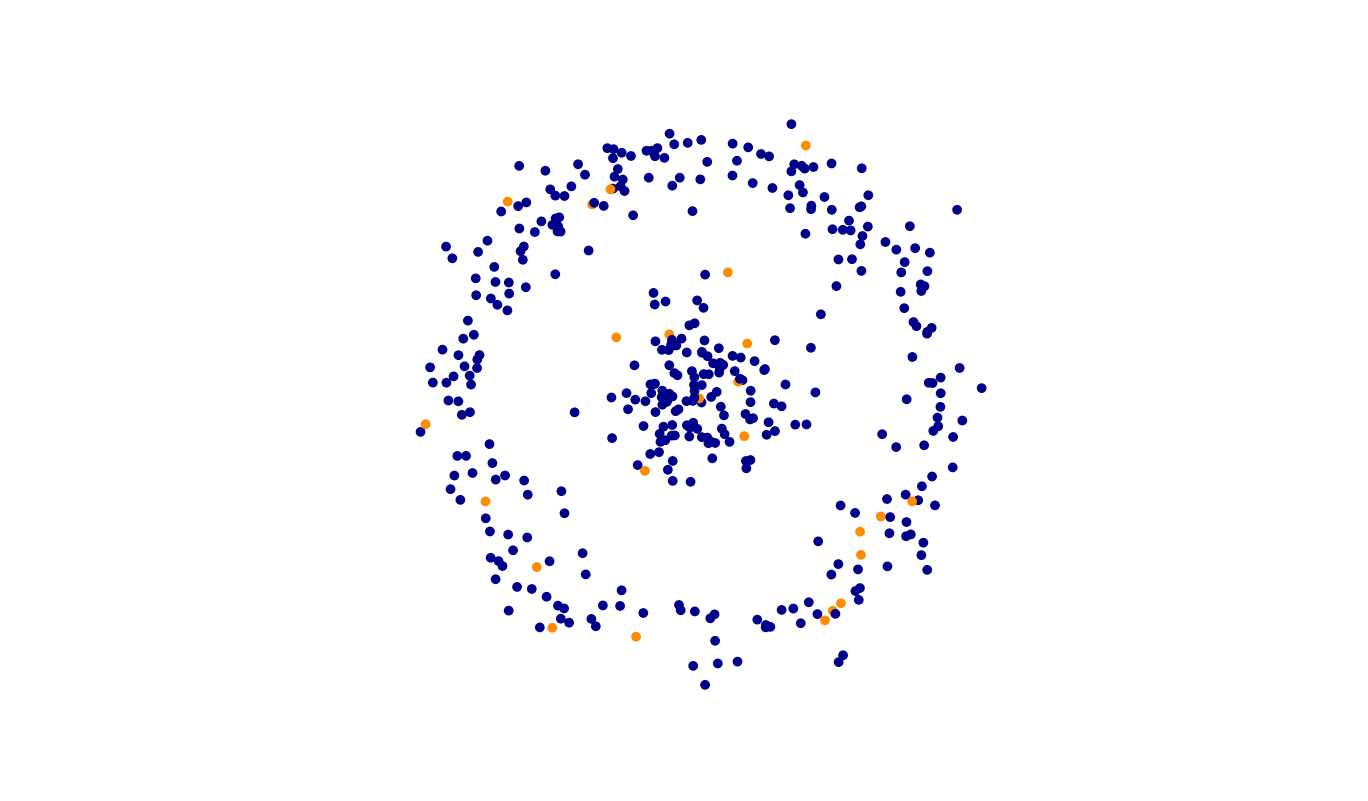}
		\end{minipage}
		\qquad
		\begin{minipage}{.4\linewidth}
			\centering
			corrupted data set
			
			\smallskip
			\includegraphics[width=\linewidth,clip=true,trim=280 80 280 80]{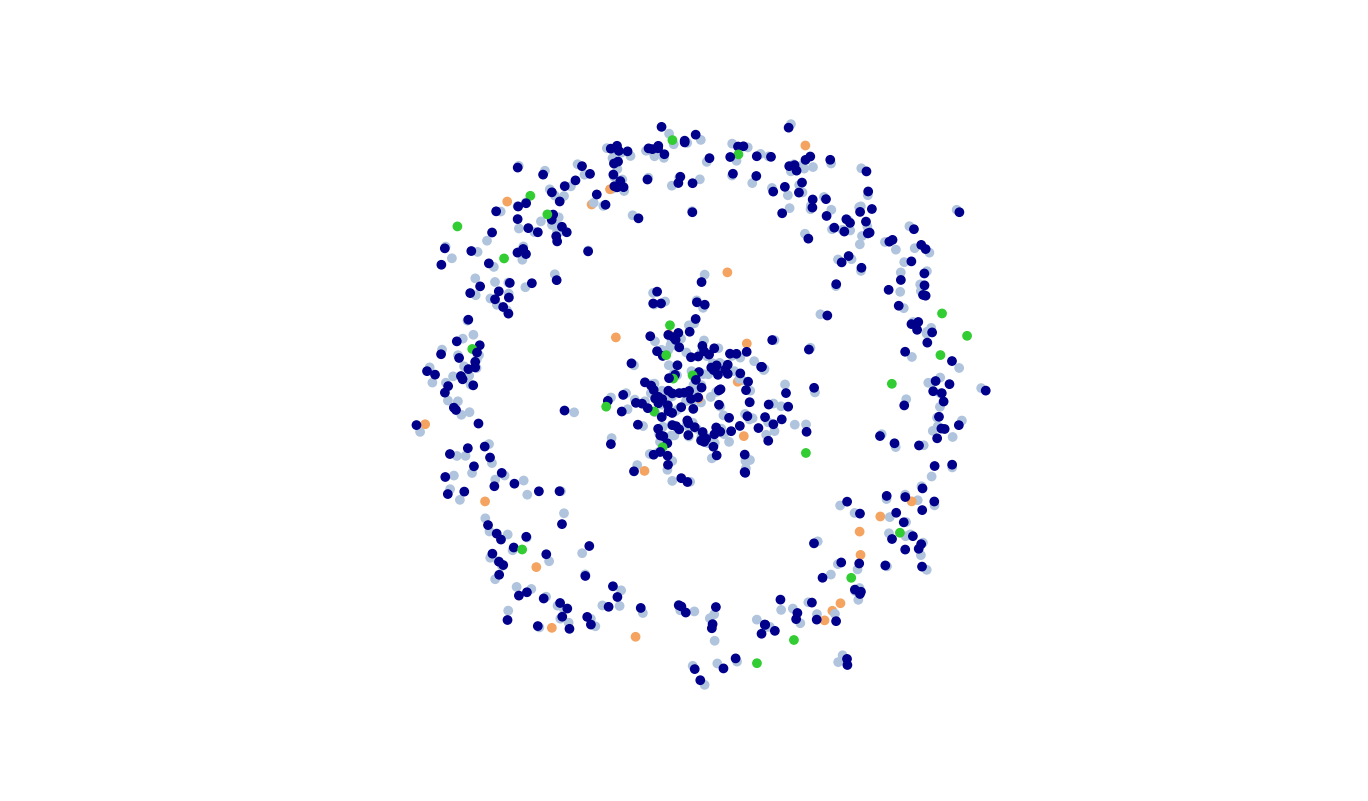}
		\end{minipage}
	
		\caption{\label{fig:pi}Illustration of the reference data set $X$ (left), with the orange data points vanishing after perturbation, and the corrupted data set $\pi(\omega)=\pi_1(\omega)\cup\pi_2(\omega)$ (right). Here $\pi_1(\omega)$ are the perturbed, non-vanishing points from $X$ (blue) and $\pi_2(\omega)$ is a set of random additional data (green). The illustration of the corrupted data set on the right shows the reference data set in transparent colors.}
	\end{figure}
		Using the same notion of similarity induced by the kernel $k$ leads to a random graph Laplacian
	\[
	\bfL_{\pi(\omega)}=\bfI_{|\pi(\omega)|}-\bfD_{\pi(\omega)}^{-1/2}\bfK_{\pi(\omega)}\bfD_{\pi(\omega)}^{-1/2}
	\]
	which can be computed for all samples, allowing to compute the eigenvector $\bfv_{\pi(\omega)}$ of the second smallest eigenvalue of $\bfL_{\pi(\omega)}$ and thus a spectral bi-clustering of $\pi(\omega)$. This random spectral clustering can be related to the clustering of $X$ through the relations \cref{eq:biclusteringX1X2} to obtain a random clustering of $X$ given through 
	\[
	A_{X|\pi(\omega)}
	\qquad
	\text{and}
	\qquad
	\overline{A_{X|\pi(\omega)}}.
	\]
	
	\subsection{Uncertainty perspective and gauging}
	In the following, we aim to quantify the uncertainty in the clustering by means of statistical quantities of interest. To this end, two perspectives emerge.
	\begin{description}
		\item[The data perspective:] The first and simpler perspective is to consider the belonging of a given data point to a given cluster $A_{X|\pi(\omega)}$.
		This makes the clustering at a given point a random variable which takes values in $\{0,1\}$.
		\item[The clustering perspective:] The second perspective is to consider the clustering itself as a set-valued random variable. This makes the clustering a random variable which takes values in the power set $\calP(X)$ of $X$.
	\end{description}
	Unfortunately, both perspectives rely on the mapping $\Omega\to\{A_{X|\pi(\omega)},\overline{A_{X|\pi(\omega)}}\}$, which is \emph{not} well defined. The reason is that the definition of the sets $A_{X|\pi(\omega)}$ and $\overline{A_{X|\pi(\omega)}}$ depends on the sign of an eigenvector or -function, cf.\ \cref{eq:biclusteringX1X2}, which is ambiguous. To make the random variables well defined we need to \emph{gauge} the sign of the eigenvector to a reference vector by proceeding as in \cref{sec:samedataset}. That is, we pick an arbitrary but ``reasonable'' $\omega^\star\in\Omega$ for which we can compute $\bfv_{X_1}(\pi(\omega^\star))$ as in \cref{eq:efevalX1} and \emph{choose} for each $\omega\in\Omega$, $\omega\neq\omega^\star$, the sign of $\bfv_{X_1}(\pi(\omega))$ such that $\langle \bfv_{X_1}(\pi(\omega)),\bfv_{X_1}(\pi(\omega^\star))\rangle_{\bbR^{|X|}}>0$. By ``reasonable'' we mean that we require that this procedure can be performed for $\bbP_\Omega$-almost all $\omega\in\Omega$. Of course, the existence of such an $\omega^\star$ is an assumption.
	\begin{definition}
		Let $(\Omega,\Sigma,\bbP_\Omega)$ be a probability space and let $\rho_Xf_{\pi(\cdot)}\colon\Omega\to\bbR^{|X|}$ be defined through \cref{eq:efevalX1} be a measurable random variable which satisfies
		\[
		\langle\rho_Xf_{\pi(\omega^\star)},\rho_Xf_{\pi(\omega)}\rangle_{\bbR^d}>0\qquad\text{for $\bbP_\Omega$-almost all}~\omega\in\Omega,
		\]
		for some fixed $\omega^\star\in\Omega$. Then we say that the clusters $A_{X|\pi(\omega)}$ and $\overline{A_{X|\pi(\omega)}}$ defined as in \cref{eq:biclusteringX1X2} are \emph{gauged} for $\bbP_\Omega$-almost all $\omega\in\Omega$.
	\end{definition}
	Having gauged clusters $A_{X|\pi(\omega)}$ and $\overline{A_{X|\pi(\omega)}}$ implies immediately that the mappings
	\begin{align*}
	\calC&\colon\Omega\to\calP(X),\hspace*{-2cm}&\omega&\mapsto A_{X|\pi(\omega)},\\
	\overline{\calC}&\colon\Omega\to\calP(X),\hspace*{-2cm}&\qquad\omega&\mapsto\overline{A_{X|\pi(\omega)}},
	\end{align*}
	are well defined and measurable mappings attaining closed subsets of $X$ as values. As we will see in the following, this allows to define statistical quantities of interest of random uncertainties in the clustering from the data perspective in a rather straightforward fashion. For statistical quantities of interest from the clustering perspective, we note that $\calC$ and $\overline{\calC}$ are \emph{random closed sets (RACS)}. The later is still an active field of research and we refer to \cite{Mol2017} for an adequate overview.
	
	In the following, we will build on some of the available notions of the RACS literature to define statistical quantities of interests. For brevity, we will restrict our considerations to $A_{X|\pi(\cdot)}$ and keep in mind that the same considerations also hold for $\overline{A_{X|\pi(\cdot)}}$.
	
	\subsection{Data perspective: the coverage function}
	As one of the canonical ways to asses uncertainties in clustering we consider the probability that a given point in the data set is contained in a given cluster. This \emph{coverage function} is given by
	\begin{equation}\label{eq:coveragefunction}
		\bbE[\mathbbm{1}_{A_{X|\pi(\cdot)}}]
		\colon
		X\to[0,1]^{|X|},
	\end{equation}
	where
	\[
	\mathbbm{1}_{C}(x)
	=
	\begin{cases}
		1,&x\in C,\\
		0,&x\notin C,
	\end{cases}
	\]
	is the indicator function and
	\[
	\bbE[\mathbbm{1}_{A_{X|\pi(\cdot)}}](x)
	\isdef
	\bbE[\mathbbm{1}_{A_{X|\pi(\cdot)}}(x)]
	=
	\int_\Omega\mathbbm{1}_{A_{X|\pi(\omega)}}(x)\dd\bbP_\Omega(\omega)
	=
	\bbP_\Omega\big(x\in A_{X|\pi(\cdot)}\big).
	\]
	
	\subsection{Data perspective: (expected) misclustering rate}
	A further way to asses uncertainties in clustering is to consider how many data points have changed their affiliation from $X$ to $\pi(\omega)$, i.e., how many data poins are \emph{misclustered}. In the language of set theory, this is precisely the cardinality of $A_X\triangle A_{X|\pi(\omega)}$, where $A\triangle B=(A\setminus B)\cup(B\setminus A)$ denotes the \emph{symmetric difference of two sets} $A,B$. Thus, we can compute the \emph{expected misclustering rate} which is given as
	\begin{equation}\label{eq:expmisclust}
		\bbE\Big[\big|A_X\triangle A_{X|\pi(\cdot)}\big|\Big]
		\in
		\bbR_{\geq 0}.
	\end{equation}
	
	\subsection{Cluster perspective: Vorob'ev expectation}
	Directly related to the coverage function from \cref{eq:coveragefunction} and the expected misclustering rate \cref{eq:expmisclust} is the Vorob'ev expectation \cite{Vor1984}. To this end, we define the level sets
	\[
	\{\bbE[\mathbbm{1}_{A_{X|\pi(\cdot)}}]\geq t\}
	=
	\{x\in X\colon\bbE[\mathbbm{1}_{A_{X|\pi(\cdot)}}](x)\geq t\}
	\subset
	X,
	\]
	$t\in[0,1]$. For
	\[
	t^{\star}=\inf\Big\{t\in[0,1]\colon|\{\bbE[\mathbbm{1}_{A_{X|\pi(\cdot)}}]\geq t\}|\leq\bbE[|A_{X|\pi(\cdot)}|]\Big\}
	\]
	the \emph{Vorob'ev expectation} $\bbE_{\text{V}}[A_{X|\pi(\cdot)}]$ of $A_{X|\pi(\cdot)}$ is defined as a set which satisfies the volume condition $|\bbE_{\text{V}}[A_{X|\pi(\cdot)}]|=\bbE[|A_{X|\pi(\cdot)}|]$ and
	\begin{equation}\label{eq:defvorobev}
		\{\bbE[\mathbbm{1}_{A_{X|\pi(\cdot)}}]>t^\star\}
		\subset
		\bbE_{\text{V}}[A_{X|\pi(\cdot)}]
		\subset
		\{\bbE[\mathbbm{1}_{A_{X|\pi(\cdot)}}]\geq t^\star\}
	\end{equation}
	holds. The Vorob'ev expectation is the minimizer of the function
	\[
	X\to\bbR_{\geq 0},\qquad M\mapsto\bbE[|A_{X|\pi(\cdot)}\triangle M|].
	\]
	Thus, it is the set around which the expected misclustering rate of $A_{X|\pi(\cdot)}$ is has the lowest value. The Vorob'ev expectation can be seen as a special case of the more general distance-average expectations, see, e.g., \cite{Mol2017}. A special case thereof with simple interpretation is the following.
	
	\subsection{Cluster perspective: oriented distance function expectation}
	For $x\in X$ and $C\subset X$ we define the distance function from $x$ to $C$ as
	\[
	d_C(x)=
	\begin{cases}
		\inf_{y\in C}d(x,y),&C\neq\emptyset,\\
		\infty,&C=\emptyset,
	\end{cases}
	\]
	and the \emph{oriented distance function (ODF)} from $x$ to $C$ as
	\[
	b_C(x)=d_{\overline{C}}(x)-d_C(x).
	\]
	Then, if $\bbE[b_{A_{X|\pi(\cdot)}}](x)<\infty$ for all $x\in X$, the \emph{ODF-expectation} is defined as
	\begin{equation}\label{eq:ODFmean}
		\bbE_{\text{ODF}}[A_{X|\pi(\cdot)}]
		=
		\Big\{
		x\in X\colon\bbE[b_{A_{X|\pi(\cdot)}}](x)\geq 0
		\Big\},
	\end{equation}
	see \cite{JS2010}\footnote{We note the different sign of $b_A$ in the reference.}.
	
	\subsection{Cluster perspective: spectral expectation}
	Motivated by the simplicity of the definition \cref{eq:ODFmean} we may note that the signs of the ODF $b_{A_{X|\pi(\omega)}}(x)$, $x\in X$, and the (gauged) continuation $\rho_Xf_{\pi(\cdot)}$ from \cref{eq:efevalX1} coincide. It is thus tempting, to propose a new expectation 
	\begin{equation}\label{eq:spectralmean}
		\bbE_{\sigma}[A_{X|\pi(\cdot)}]
		=
		\Big\{
		x\in X\colon (\bbE[\rho_Xf_{\pi(\cdot)}])_x\geq 0
		\Big\}
	\end{equation}
	to which we will refer to as the \emph{spectral expectation} in the following.
	We note that the linearity of the mean implies that we can replace $\bbE[\rho_Xf_{\pi(\cdot)}]$ by $\rho_X\bbE[f_{\pi(\cdot)}]$.
	
	\section{Monte Carlo estimators for expectations of clusterings}\label{sec:mc}
	In most cases, the mean of the quantities of interest in  the previous section is not available in closed form but needs to be approximated. In the following, we rely on the Monte Carlo estimator for Banach space-valued quantities, i.e.,
	\[
	\bbE[g]
	\approx
	\frac{1}{M}
	\sum_{i=1}^Mg(\omega_i)
	\defis
	E_M[g]
	\]
	for any integrable mapping $g\colon\Omega\to\calY$, $\calY$ a Banach space, to approximate the coverage function \cref{eq:coveragefunction} by $E_M[\mathbbm{1}_{A_{X|\pi(\cdot)}}]\approx\bbE[\mathbbm{1}_{A_{X|\pi(\cdot)}}]$ and the expected misclustering rate \cref{eq:expmisclust} by $E_M\big[|X\triangle\pi(\cdot)|\big]\approx\bbE\big[|X\triangle\pi(\cdot)|\big]$. We recall that the Monte Carlo estimator is an unbiased estimator.
	
	The ODF-expectation \cref{eq:ODFmean} and the spectral expectation \cref{eq:spectralmean} can likewise be approximated using the Monte Carlo estimator by estimating $\bbE[b_{A_{X|\pi(\cdot)}}]$ and $\bbE[\rho_Xf_{\pi(\cdot)}]$, respectively. This yields the empirical estimators
	\[
	E_{M,\text{ODF}}[A_{X|\pi(\cdot)}]
	=
	\Big\{
	x\in X\colon E_M[b_{A_{X|\pi(\cdot)}}](x)\geq 0
	\Big\},
	\]
	and
	\[
	E_{M,\sigma}[A_{X|\pi(\cdot)}]
	=
	\Big\{
	x\in X\colon (E_M[\rho_Xf_{\pi(\cdot)}])_x\geq 0
	\Big\}.
	\]
	For the Vorob'ev expectation, we follow the the procedure by \cite{HST2012,Kov1986}. Introducing the empirical cardinalities
	\[
	\Gamma_M=\frac{1}{M}\sum_{i=1}^M|A_{X|\pi(\omega_i)}|
	\]
	and setting
	\[
	F_M(t)=|\{E_M[\mathbbm{1}_{A_{X|\pi(\cdot)}}]\geq t\}|
	\]
	and
	\[
	t^{\star}_n=\inf\Big\{t\in[0,1]\colon F_M(t)\leq\Gamma_M\Big\},
	\]
	the \emph{Kovyazin mean} $K_M[\mathbbm{1}_{A_{X|\pi(\cdot)}}]$, is a Borel set which is given with respect to an empirical version of \cref{eq:defvorobev}, i.e., a set such that $|K_M[\mathbbm{1}_{A_{X|\pi(\cdot)}}]|=\Lambda_M$ and
	\[
	\{E_M[\mathbbm{1}_{A_{X|\pi(\cdot)}}]>t^\star\}
	\subset
	K_M[A_{X|\pi(\cdot)}]
	\subset
	\{E_M[\mathbbm{1}_{A_{X|\pi(\cdot)}}]\geq t^\star\}.
	\]
	
	\section{Consistency analysis}\label{sec:consistency}
	
	\subsection{Consistency of spectral clustering}
	In the seminal paper \cite{VBB2008} it was shown that $\calU_X$ from \cref{eq:discreteIO} converges for $|X|\to\infty$ in a reasonable sense to a limit operator $\calU$ if some additional assumptions are made. To facilitate our understanding, we will write $X_n=X$ with $n=|X_n|=|X|$ in the following.
	\begin{assumption}[{\cite[General assumption]{VBB2008}}]\label{ass:VBB2008GA}
		Assume that $\calX$ is a compact metric space, $\calB$ the Borel $\sigma$-algebra on $\calX$, and $\bbP_\calX$ a probability measure on $(\calX,\calB)$. Assume without loss of generality that the support of $\bbP_\calX$ coincides with $\calX$ and that the data points $\{x_i\}_{i\in\bbN}$ are iid random variables, each with measure $\bbP_\calX$. Assume that the similarity function $k\colon\calX\times\calX\to\bbR$ is symmetric, continuous, and bounded away from zero, i.e., assume that there is a constant $\ell>0$ such that $k(x,y)>\ell$ for all $x,y\in\calX$.
	\end{assumption}
	\Cref{ass:VBB2008GA} allows us to introduce data-continuous versions of the operators from \cref{eq:continuouspartitioning}. To this end, we introduce the \emph{continuous degree function}
	\[
	d(x)=\int_\calX k(x,y)\dd\bbP_\calX(y),
	\]
	the \emph{continuous normalized similarity function}
	\[
	h(x,y)=\frac{k(x,y)}{\sqrt{d(x)d(y)}},
	\]
	and the integral operator
	\[
	\calT\colon \calC(\calX)\to \calC(\calX),\qquad \calT f(x)=\int_\calX h(x,y)f(y)\dd\bbP_\calX(y),
	\]
	to define
	\begin{equation}\label{eq:IO}
	\calU=I-\calT.
	\end{equation}
	The following relation holds between the spectrum of $\calU_X$ and the spectrum of $\calU$.
	
	\begin{theorem}[{\cite[Theorem 15]{VBB2008}}]\label{thm:VBB2008thm15}
		Assume that \cref{ass:VBB2008GA} holds and write $X_n=\{x_1,\ldots,x_n\}\subset\calX$. Let $\lambda\neq 1$ be an eigenvalue of $\calU$ and $M\subset\bbC$ an open neighbourhood of $\lambda$ such that $M\cap\sigma(\calU)$, with the spectrum $\sigma(\calU)$ of $\calU$, only contains $\lambda$. Then
		\begin{enumerate}
			\item Convergence of eigenvalues: The eigenvalues in $\sigma(\calU_{X_n})\cap M$ converge to $\lambda$ in the sense that every sequence $\{\lambda_n\}_{n\in\bbN}$ with $\lambda_n\in\sigma(\calU_{X_n})\cap M$ satisfies $\lambda_n\to\lambda$ almost surely.
			\item Convergence of eigenvectors: If $\lambda$ is a simple eigenvalue, then the eigenvectors of $\calU_{X_n}$ converge almost surely up to a change of sign: if $f_n$ is the eigenfunction of $\calU_{X_n}$ with eigenvalue $\lambda_n$, and $f$ the eigenfunction to $\lambda$, then there exists a sequence $\{a_n\}_{n\in\bbN}\in\{\pm1\}^\bbN$ (the signs of the eigenvectors) such that
			\[
			\|a_nf_n-f\|_{C(\calX)}\to 0\qquad\text{as}~n\to\infty
			\]
			almost surely. Moreover, the sets $\{a_nf_{X_n}>0\}$ and $\{f>0\}$ converge, that is, their symmetric difference satisfies
			\[
			\bbP_\calX\big(\{f\geq 0\}\triangle\{a_nf_{X_n}\geq 0\}\big)\to 0,\qquad\text{as}~n\to\infty.
			\]
		\end{enumerate}
	\end{theorem}
	While the above behaviour is a quantitative convergence result, we will require in the following a qualitative convergence result. In the following we recall such a quantitative result for the case when $k$ is the Gaussian kernel.
	\begin{theorem}[{\cite[Example 1]{VBB2008}}]\label{thm:uniformconv}
		Let $\calX$ be a compact subset of $\bbR^d$ and $k(x,y)=\exp(-\|x-y\|_{\bbR^d}^2/(2\sigma^2)))$, $\sigma>0$.
		Under the assumptions of \cref{thm:VBB2008thm15} there exists a constant $C>0$ which is independent of $X_n$ such that
		\[
		\|a_nf_n-f\|_{C(\calX)}\leq Cn^{-1/2}.
		\]
	\end{theorem}
	
	\subsection{Consistency of Monte Carlo estimators}
	In the following we summarize known consistency estimators for the quantities from \cref{sec:expectations} in a common setting. To this end, we make the following assumption.
	\begin{assumption}\label{ass:mcassum}
		Let $(\calX,\calB,\mu_\calX)$ be a measure space and $g\colon \calX\to\bbR$ a measurable function. Let $(\Omega,\Sigma,\bbP_\Omega)$ be a probability space and $g_m=g_m(\omega,\cdot)$, $\omega\in\Omega$, a sequence of random functions $g_m\colon M\to\bbR$, $m\in\bbN$. Assume that for each $m$, with probability one, $g_m=g_m(\omega,\cdot)$ is measurable.
	\end{assumption}
	
	From the strong law of large numbers we know that the Monte Carlo estimator applied to random variables taking values Banach spaces is a consistent estimator. More precisely, we have the following result.
	\begin{theorem}[Strong law of large numbers]\label{thm:LLN}
		Let \cref{ass:mcassum} hold and let $g_m,g\in L_{\mu_\calX}^1(\Omega;F)$, $m\in\bbN$, $F$ a Banach space, be independent, identically distributed random variables with $\bbE[g_m]=g$, $m\in\bbN$. Write
		\[
		E_M[g]=\frac{1}{M}\sum_{m=1}^Mg_m.
		\]
		Then it holds
		\[
		\|E_M[g]-g\|_{F}\to 0
		\qquad
		\text{as}~M\to\infty
		\]
		$\bbP_\Omega$-almost surely.
	\end{theorem}
  	In our setting, the law of large numbers immediately implies that
	\[
  		E_M[\mathbbm{1}_{\{g\geq0\}}]\to\bbE[\mathbbm{1}_{\{g\geq0\}}]
  	\]
	uniformly for $\bbP_\calX$-almost every $x\in\calX$ as $M\to\infty$ and $\bbP_\Omega$-almost surely. We will use this result below to show consistency of the approximations for the coverage function.
	
	Despite the lack of suitable linear relations, it is still possible to derive laws of large numbers for set-valued random variables. For the ODF expectation and the spectral expectation, the following result is helpful.
	\begin{theorem}[{\cite[Theorem 3]{CGR2006}}]\label{thm:cuevathm}
		Let the assumptions of \cref{ass:mcassum} hold with $g,g_m\in L^p_{\mu_\calX}(\calX)$ for some $1\leq p\leq\infty$. Assume that
		\[
		\mu_\calX(\{g=0\})=0
		\]
		and
		\[
		\mu_\calX(\{g\in[-\varepsilon_0,\varepsilon_0)\})<\infty
		\]
		for some $\varepsilon_0$. Then $\|g-g_m\|_{L^p_{\mu_\calX}(\calX)}\to 0$ $\bbP_\Omega$-almost surely implies
		\[
		\mu_\calX(\{g\geq 0\}\triangle \{g_m\geq 0\})\to 0\qquad\text{as}~m\to \infty
		\]
		$\bbP_\Omega$-almost surely.
	\end{theorem}
	Based on this theorem, a consistency result for the Monte Carlo method without additional discretization error was proven in \cite{JS2012}.
	
	A consistency result for the Kovyazin mean was derived in \cite{Kov1986} for  Monte Carlo sampling and in \cite{HST2012} for Monte Carlo sampling with additional grid discretization. However, it is not immediate how these techniques can be applied in the present setting.
	
	\subsection{Consistency of Monte Carlo estimators for spectral clustering}
	With the preliminaries in place, we will in the following provide a theorem on the consistency of the Monte Carlo estimators in the large data limit, assuming that we have no missing or additional data.
	\begin{assumption}\label{ass:randomdata}
		Assume that \cref{ass:VBB2008GA} holds and let $\tilde{X}_n=\{x_1,\ldots,x_n\}\subset\calX$, $n\in\bbN$. Let $(\Omega,\Sigma,\bbP_\Omega)$ be a probability space and $\omega\in\Omega$ independent of $x_i\in\calX$, $i\in\bbN$. Let $\varepsilon\colon\calX\times\Omega\to\calX$ such that for $\bbP_\Omega$-almost every $\omega\in\Omega$ the support of its pushforward measure $\bbP_{\calX,\omega}=\bbP_\calX\circ\varepsilon(\cdot,\omega)^{-1}$ coincides with $\calX$, and
		\[
		X_n(\omega)
		=
		\{\varepsilon(x,\omega)\colon x\in\widetilde{X}_n\}
		\]
		be the data sets subject to measurement noise consisting of the iid random variables $\varepsilon(x_i,\omega)$, $i=1,\ldots,n$, distributed according to $\bbP_{\calX,\omega}$.
	\end{assumption}
	In the following, we will also require the analogs of \cref{eq:IO} subject to measurement noise. To this end, we introduce
	\begin{align*}
	d(x,\omega)&=\int_\calX k(x,y)\dd\bbP_{\calX,\omega}(y),\\
	h(x,y,\omega)&=\frac{k(x,y)}{\sqrt{d(x,\omega)d(y,\omega)}},
	\end{align*}
	the integral operator
	\[
	\calT(\omega)\colon \calC(\calX)\to \calC(\calX),\qquad \calT(\omega) f(x)=\int_\calX h(x,y,\omega)f(y)\dd\bbP_{\calX,\omega}(y),
	\]
	and
	\[
	\calU(\omega)=I-\calT(\omega).
	\]
	The following assumption is to ensure that the considered eigenpairs $(\lambda(\omega),f(\omega))$ of the randomly perturbed operators $\calU(\omega)$ can be put in a meaningful relation to the ones of $\calU$.
	\begin{assumption}\label{ass:random}
		Assume that \cref{ass:VBB2008GA} and \cref{ass:randomdata} hold. Assume further that
		\begin{enumerate}
			\item $\lambda\neq 1$ is an eigenvalue of multiplicity one of $\calU$ and $f$ its associated eigenfunction. Assume that $M_1,M_2\subset\bbC$ are open subset with positive distance such that $M_1$ is an open neighbourhood of $\lambda$ and that $M_2\subset\bbC$ is an open neighbourhood of $\sigma(\calU)\setminus\{\lambda\}$. 
			\item $\bbP_\Omega$-almost all randomly perturbed operators $\calU(\omega)$ have an eigenvalue $\lambda(\omega)$ of multiplicity one contained in $M_1$ with corresponding eigenfunction $f(\omega)$.
			Assume that $\sigma(\calU(\omega))\setminus\{\lambda(\omega)\}\subset M_2$.
		\end{enumerate}
	\end{assumption}
	An immediate consequence is the following.
	\begin{corollary}\label{cor:efuniformconv}
		Let $\calX$ be a compact subset of $\bbR^d$ and consider the similarity function $k(x,y)=\exp(-\|x-y\|_{\bbR^d}^2/(2\sigma^2)))$ with $\sigma>0$. Let \cref{ass:VBB2008GA}, \cref{ass:randomdata}, and \cref{ass:random} hold. Then the eigenvectors of $\calU_{X_n(\omega)}$ converge almost surely up to a change of sign: if $f_n(\omega)$ is the eigenfunction of $\calU_{X_n(\omega)}$ with eigenvalue $\lambda_n(\omega)$, and $f(\omega)$ the eigenfunction to $\lambda(\omega)$, then there exists a sequence $\{a_n(\omega)\}_{n\in\bbN}\in\{\pm1\}^\bbN$ (the signs of the eigenvectors) and a constant $C>0$ which is independent of $X_n(\omega)$ and $\omega$ such that
		\[
		\|a_n(\omega)f_n(\omega)-f(\omega)\|_{C(\calX)}\leq Cn^{-1/2},
		\]
		$\bbP_{\calX,\omega}$-almost surely and for $\bbP_\Omega$-almost all $\omega\in\Omega$.
	\end{corollary}
	\begin{proof}
		Follows from a close investigation of the proof of \cref{thm:uniformconv}, the details of which are given in \cite[Chapter 6]{VBB2008} and \cite[Theorem 3]{Atk1967}, and an eigenvalue perturbation argument \cite{Kat1995}.
	\end{proof}
	With all the ingredients available, we prove in the following that continuous analogs of the finite data empirical mean approximations, i.e., Monte Carlo estimators, of the eigenfunctions, the coverage function, the ODF-expectation, and the spectral expectation are consistent estimators towards the infinite data point and Monte Carlo sample limits. Moreover, the expected misclustering rate is a consistent estimator towards the infinite Monte Carlo sample limit. Due to its technical complexity a result for the Vorob'ev expectation is beyond the scope of this paper and left as future research.
	\begin{theorem}\label{thm:consistency}
		Let $\calX$ be a compact subset of $\bbR^d$ and consider the similarity function $k(x,y)=\exp(-\|x-y\|_{\bbR^d}^2/(2\sigma^2)))$ with $\sigma>0$. Let \cref{ass:VBB2008GA}, \cref{ass:randomdata}, and \cref{ass:random} hold. Write
		\[
		A_{\calX|\pi_n(\omega)}=\{a_n(\omega)f_n(\omega)\geq 0\},
		\quad
		A_{\calX(\omega)}=\{f(\omega)\geq 0\}.
		\]
		Then it holds
		\begin{align}
			\|E_M[a_nf_n]-\bbE[f]\|_{C(\calX)}&\overset{n,M\to\infty}{\longrightarrow} 0,&&\bbP_\calX\otimes\bbP_\Omega-\text{a.s.},\label{eq:efconv}\\
			\big\|E_M[\mathbbm{1}_{A_{\calX|\pi_n(\cdot)}}]-\bbE[\mathbbm{1}_{A_{\calX(\cdot)}}]\big\|_{L_{\bbP_\Omega\otimes\bbP_\calX}^1(\Omega\times\calX)}&\overset{n,M\to\infty}{\longrightarrow} 0,&&\bbP_\calX-\text{a.s.},\label{eq:covconf}
			\intertext{and, if $\bbP_{\Omega}(\{\bbE[b_{\{f\geq 0\}}]=0\})=0$,}
			\bbP_\calX\Big(E_{M,\text{ODF}}\big[A_{\calX|\pi_n(\cdot)}\big]\triangle \bbE_{\text{ODF}}\big[A_{\calX(\cdot)}\big]\Big)&\overset{n,M\to\infty}{\longrightarrow} 0,&&\bbP_\calX\otimes\bbP_\Omega-\text{a.s.},\label{eq:ODFconv}
			\intertext{and, if $\bbP_{\Omega}(\{\bbE[f]=0\})=0$,}
			\bbP_\calX\Big(E_{M,\sigma}\big[A_{\calX|\pi_n(\cdot)}\big]\triangle \bbE_{\sigma}\big[A_{\calX(\cdot)}\big]\Big)&\overset{n,M\to\infty}{\longrightarrow} 0,&&\bbP_\calX\otimes\bbP_\Omega-\text{a.s.}\label{eq:specconv}
		\end{align}
		The limits are independent of the order in which they are taken. Moreover, it holds
		\begin{align}\label{eq:misclrate}
			E_M\big[|A_{X_n}\triangle A_{X_n|\pi_n(\cdot)}|\big]
			\overset{M\to\infty}{\longrightarrow}
			\bbE\big[|A_{X_n}\triangle A_{X_n|\pi_n(\cdot)}|\big],&&\bbP_\calX\otimes\bbP_\Omega-\text{a.s.}
		\end{align}
	\end{theorem}
	\begin{remark}
		For clarity we recall the dependencies of the quantities interest considered in \cref{thm:consistency}. On the one hand, the clustering $A_{\calX(\omega)}\subset\calX$ given through $f(\omega)\colon\calX\to\bbR$ depends on the random perturbations modelled through $\omega\in\Omega$ with measure $\bbP_\Omega$. On the other hand, the approximated quantities have an additional dependence on the data points $\{x_i\}_{i\in\bbN}$, cf.\ \cref{ass:VBB2008GA}, each of them being an element of $\calX$ and having measure $\bbP_\calX$. The $\bbP_\calX$- and $\bbP_\calX\otimes\bbP_\Omega$-almost sure statements are to be considered with respect to the sampled data points and the random perturbations.
	\end{remark}
	\begin{proof}[Proof of \cref{thm:consistency}]
		\begin{description}
			\item[\Cref{eq:efconv}:]
			\Cref{thm:LLN} and \cref{cor:efuniformconv} imply
			\begin{align*}
			\|E_M[a_nf_n]-\bbE[f]\|_{C(\calX)}
			&\leq
			\|E_M[a_nf_n]-E_M[f]\|_{C(\calX)}
			+
			\|E_M[f]-\bbE[f]\|_{C(\calX)}
			\\
			&\leq
			\frac{1}{M}\sum_{m=1}^M\|a_n(\omega_i)f_n(\omega_i)-f(\omega_i)\|_{C(\calX)}\\
			&\qquad\qquad\qquad\qquad\qquad\qquad\,\,\,
			+\|E_M[f]-\bbE[f]\|_{C(\calX)}\\
			&\leq
			C\big(n^{-1/2}+M^{-1/2}\big),
			\end{align*}
			$\bbP_\calX$- and $\bbP_\Omega$-almost surely, from which the assertion follows.
			\item[\Cref{eq:covconf}:] We write
			\begin{align*}
				&\big\|E_M[\mathbbm{1}_{A_{\calX|\pi_n(\cdot)}}]-\bbE[\mathbbm{1}_{A_{\calX(\cdot)}}]\big\|_{L_{\bbP_\Omega\otimes\bbP_\calX}^1(\Omega\times\calX)}\\
				&\qquad\leq
				\underbrace{\big\|E_M[\mathbbm{1}_{A_{\calX|\pi_n(\cdot)}}]-E_M[\mathbbm{1}_{A_{\calX(\cdot)}}]\big\|_{L_{\bbP_\Omega\otimes\bbP_\calX}^1(\Omega\times\calX)}}_{=:(\spadesuit)}\\
				&\qquad\qquad\qquad+
				\underbrace{\big\|E_M[\mathbbm{1}_{A_{\calX(\cdot)}}]-\bbE[\mathbbm{1}_{A_{\calX(\cdot)}}]\big\|_{L_{\bbP_\Omega\otimes\bbP_\calX}^1(\Omega\times\calX)}}_{=:(\clubsuit)}
			\end{align*}
			and estimate
			\begin{align*}
				(\spadesuit)
				&\leq
				\frac{1}{M}\sum_{m=1}^M\big\|\mathbbm{1}_{A_{\calX|\pi_n(\omega_m)}}-\mathbbm{1}_{A_{\calX(\omega_m)}}\big\|_{L_{\bbP_\Omega\otimes\bbP_\calX}^1(\Omega\times\calX)}\\
				&=
				\int_{\Omega}\int_\calX\big|\mathbbm{1}_{A_{\calX|\pi_n(\omega_m)}}-\mathbbm{1}_{A_{\calX(\omega_m)}}\big|\dd\,(\bbP_\Omega\otimes\bbP_\calX)\\
				&=
				(\bbP_\Omega\otimes\bbP_\calX)\big(\{a_nf_n(\omega_m)\geq 0\}\triangle\{f(\omega_m)\geq 0\}\big)
			\end{align*}
			Now, \cref{cor:efuniformconv} implies
			\begin{align*}
				\{a_n(\omega_m)f_n(\omega_m)\geq 0\}
				&=
				\{a_n(\omega_m)f_n(\omega_m)-f(\omega_m)+f(\omega_m)\geq 0\}\\
				&=
				\{f(\omega_m)\leq f(\omega_m)-a_n(\omega_m)f_n(\omega_m)\}\\
				&\subset
				\{f(\omega_m)\leq Cn^{-1/2}\}
			\end{align*}
			and, likewise,
			\[
				\{a_n(\omega_m)f_n(\omega_m)\leq 0\}
				\subset
				\{f(\omega_m)\geq -Cn^{-1/2}\}
			\]
			with a constant $C$ which is independent of $n$ and $m$, which yields
			\[
			\{a_n(\omega_m)f_n(\omega_m)\geq 0\}\triangle\{f(\omega_m)\geq 0\}
			\subset
			\{|f(\omega_m)|\leq Cn^{-1/2}\}.
			\]
			This yields
			\begin{align*}
			(\spadesuit)
			&\leq
			(\bbP_\Omega\otimes\bbP_\calX)\big(\{a_nf_n(\omega_m)\geq 0\}\triangle\{f(\omega_m)\geq 0\}\big)\\
			&\leq
			(\bbP_\Omega\otimes\bbP_\calX)\big(\{|f(\omega_m)|\leq Cn^{-1/2}\}\big)\\
			&\overset{n\to0}{\longrightarrow}
			0,
			\end{align*}
			independently of $m$. Exploiting $(\bbP_\Omega\otimes\bbP_\calX)(\Omega\times\calX)=1$ and referring to the standard Monte Carlo convergence rate of Hilbert space-valued random variables in $L^2$, we obtain
			\begin{align*}
				(\clubsuit)
				&\leq
				\big\|E_M[\mathbbm{1}_{A_{\calX(\omega_m)}}]-\bbE[\mathbbm{1}_{A_{\calX(\omega_m)}}]\big\|_{L_{\bbP_\Omega\otimes\bbP_\calX}^2(\Omega\times\calX)}\\
				&\leq
				M^{-1/2}\big\|\mathbbm{1}_{A_{\calX(\omega_m)}}\big\|_{L_{\bbP_\Omega\otimes\bbP_\calX}^2(\Omega\times\calX)}\\
				&\leq
				M^{-1/2}.
			\end{align*}
			This yields the assertion.
			\item[\Cref{eq:ODFconv}:] \Cref{eq:efconv} implies that $g_{n,M}=b_{\{E_M[a_nf_n]\geq 0\}}$ converges to $g=b_{\{\bbE[f]\geq 0\}}$ in $L_{\bbP_\calX}^{\infty}(\calX)$, i.e.,
			\[
			\|g_{n,M}-g\|_{L_{\bbP_\calX}^{\infty}(\calX)}
			\overset{n,M\to\infty}{\longrightarrow}
			0,
			\]
			$\bbP_\calX$- and $\bbP_\Omega$-almost surely. The assertion follows from \cref{thm:cuevathm}.
			\item[\Cref{eq:specconv}:] Follows likewise from \cref{thm:cuevathm}, but with $g_{n,M}=E_M[a_nf_n]$ and $g=\bbE[f]$.
			\item[\Cref{eq:misclrate}:] Follows directly from \cref{thm:LLN}.
		\end{description}
	\end{proof}
	
	\section{Numerical experiments}\label{sec:experiments}
	
	\subsection{Computational setup}
	The following numerical examples have been implemented within \verb|python| with a similarity function $k\colon\bbR^d\times\bbR^d\to\bbR$ given through $k(x,y)=\exp(\|x-y\|_{\bbR^d}^2/(2\sigma^2))$,
	where $\sigma$ is in the order of magnitude of the largest edge in a minimal spanning tree on the data set computed with the \verb|NetworkX| library \cite{HSS2008c}. Otherwise, all computations have been implemented as described in \cref{sec:preliminaries} to \cref{sec:mc}. All computations have been carried out on a single core of a Intel Xeon Gold 6136 CPU with 3GHz.

	\subsection{Point cloud in circle}\label{sec:pcic}
	
	For our first example we generate a reference data set with a total of $n=3m\in\{400,1600\}$ artificial data points in $\bbR^2$ in two clusters. The first cluster is given by $m$ points sampled according to a standard Gaussian distribution $\calN(\mathbf{0},\bfI)$, whereas the $2m$ data points from the second cluster are sampled according to $(r\cos(\phi),r\sin(\phi))$ with $r\sim\calN(2.5,0.25)$ and $\phi\sim\calU([0,2\pi))$. The sampled reference data set is illustrated in \cref{fig:pcicillu}.
	\begin{figure}
		\centering
		\includegraphics[width=0.4\textwidth]{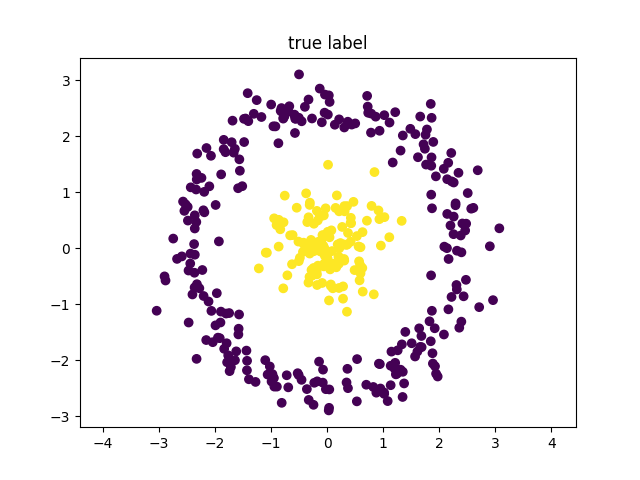}
		\includegraphics[width=0.4\textwidth]{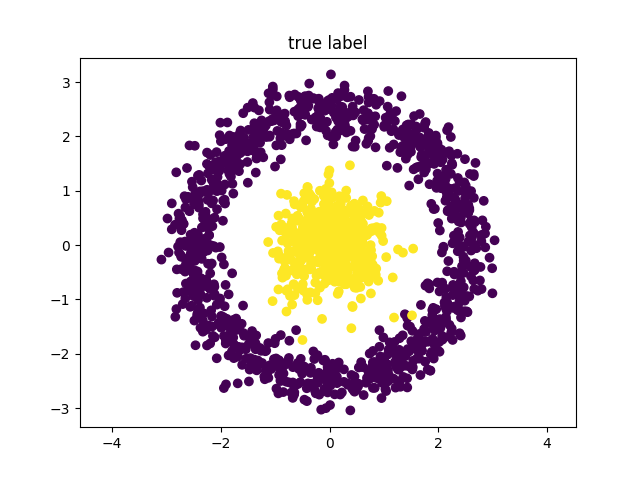}
		\caption{Sampled reference data sets of the point cloud in circle data set for $n=400$ (left) and $n=1600$ (right).}
		\label{fig:pcicillu}
	\end{figure}
	From these $n$ points, we select, delete, and regenerate in each cluster a $\calU([1\%,7\%])$-random percentage of randomly selected points from the same distributions and perturb the remaining ones by additive Gaussian noise with distribution $\calN(\mathbf{0},\varepsilon^2\bfI)$, $\varepsilon=\frac{2^k}{10}$, $k\in\{0,\pm1,\pm2\}$.
	
	For the computation of the expectations we use $10^4$ Monte Carlo samples, which takes roughly a day for the data set with $n=400$ points and four days for the data set with $n=1600$ points. The obtained expected misclustering rate is illustrated in \cref{fig:pcicmisc}, whereas the corresponding coverage functions, Vorob'ev expectation, ODF-expectation, and spectral expectation are depicted in \cref{fig:pcipexp400} for $n=400$ and \cref{fig:pcipexp1600} for $n=1600$.
	\begin{figure}
		\centering
		\includegraphics[width=0.4\textwidth]{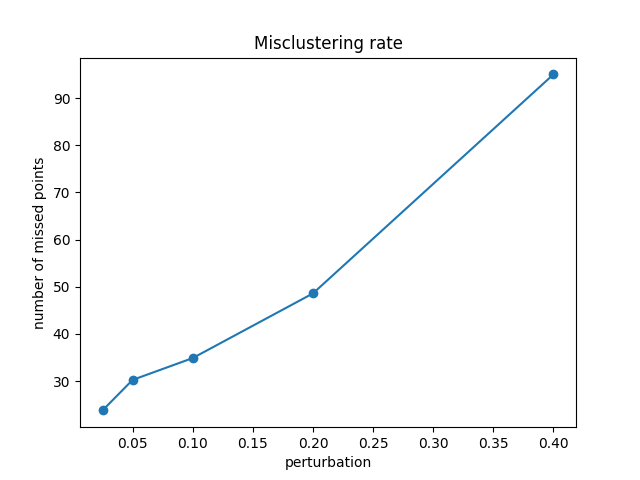}
		\includegraphics[width=0.4\textwidth]{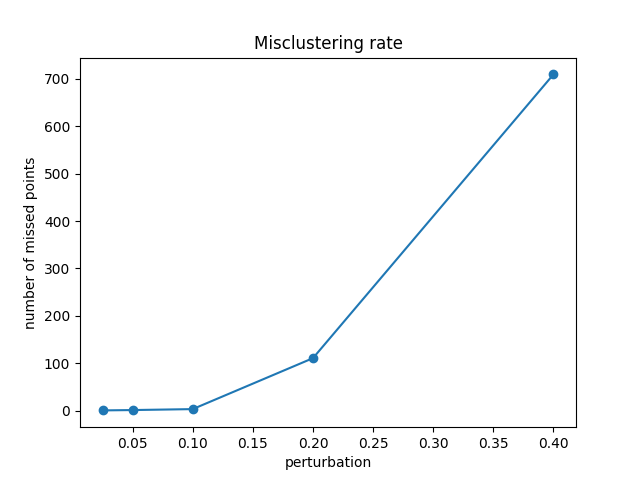}
		\caption{Expected misclustering rate for the point cloud in circle data set for $n=400$ (left) and $n=1600$ (right).}
		\label{fig:pcicmisc}
	\end{figure}
	\begin{figure}
		\centering
		\includegraphics[width=\textwidth]{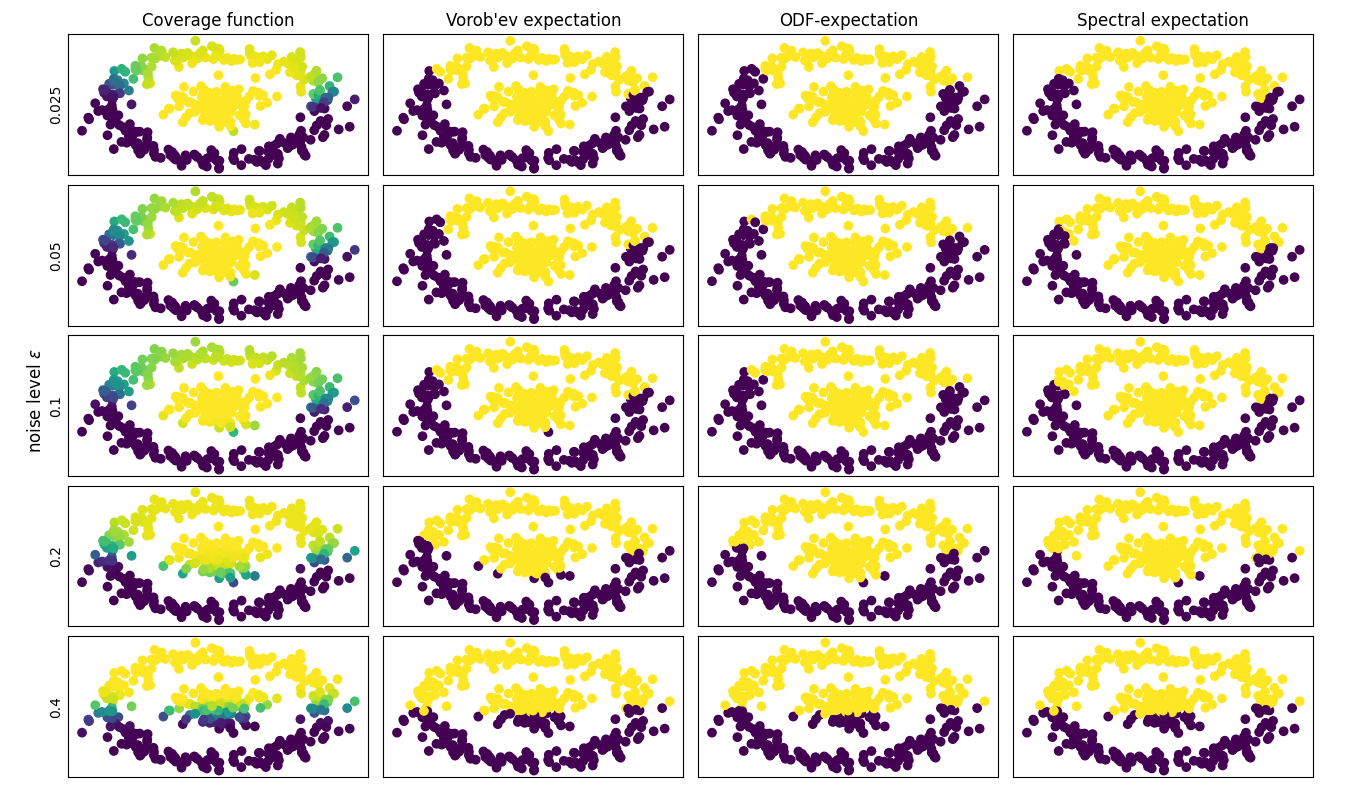}
		\caption{Approximated coverage function, Vorob'ev expectation, ODF-expectation, and spectral expectation for different noise levels for the point in circle data set with $n=400$.}
		\label{fig:pcipexp400}
	\end{figure}
	\begin{figure}
	\centering
	\includegraphics[width=\textwidth]{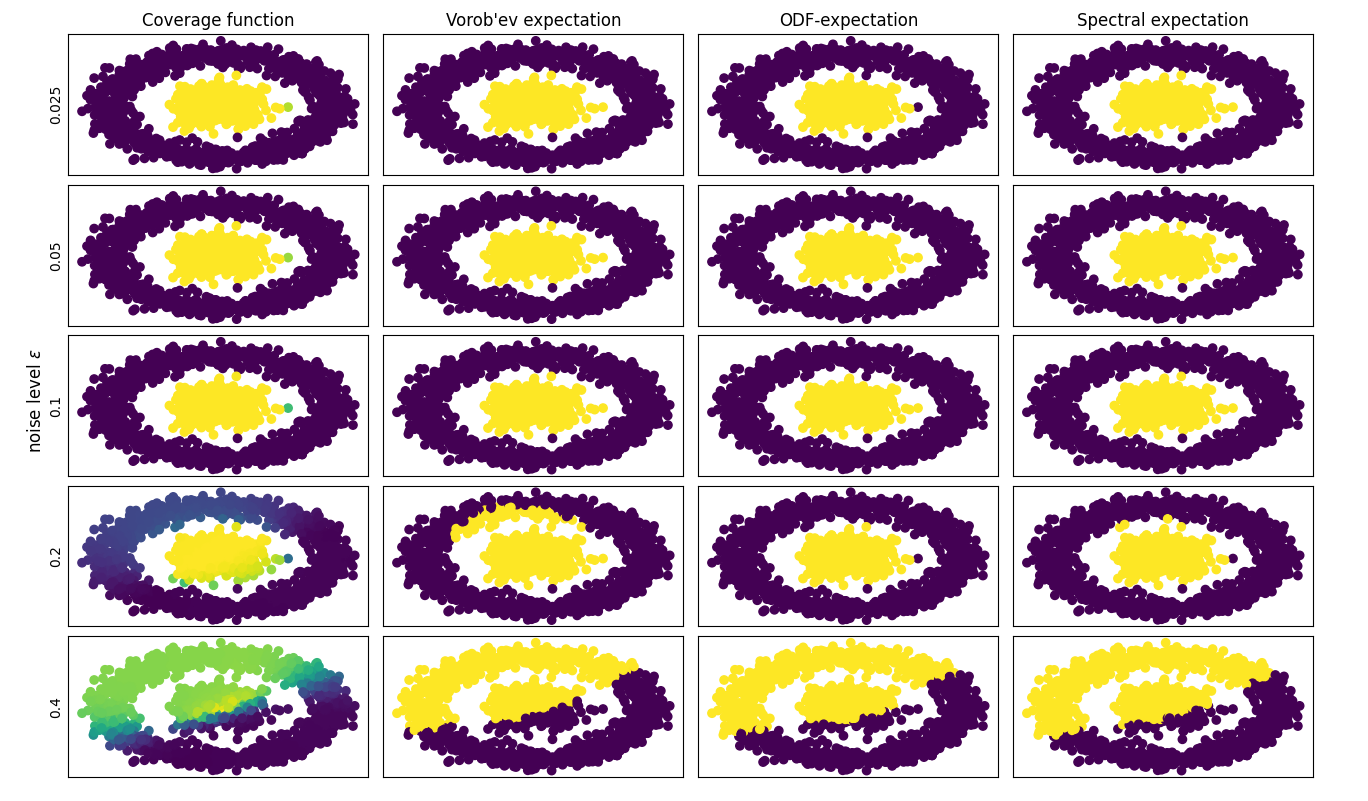}
	\caption{Approximated coverage function, Vorob'ev expectation, ODF-expectation, and spectral expectation for different noise levels for the point in circle data set with $n=1600$.}
	\label{fig:pcipexp1600}
	\end{figure}
	From the figures it is visible that, as to be expected, the larger the noise level of the data point perturbations, the larger the changes to the expected clustering. On the other hand, comparing the $n=400$ to the $n=1600$ case, it seems that the data set with more points is more stable to perturbations in the data. If the perturbations in the data is too large, the statistically expected clustering has almost no relation to the clustering of the reference data set.

	\subsection{Entangled half circles}\label{sec:ehc}
	For the second example we generate reference data set with $n=2m\in\{400,1600\}$ data points in $\bbR^2$ in two clusters. The first cluster corresponds to $n$ points $(x,y)$ with $x\sim\calU([0,\pi))$ and $y= 0.1-1.3 \cdot \sin(x) + z_1$ with $z_1\sim\calN(0,0.2)$. The second cluster corresponds to $n$ points $(x,y)$ with $x\sim\calU([0.4 \pi, 1.4\pi))$ and $y= 1.3 \cdot \sin (x-0.4\cdot \pi ) + z_2$ and $z_2\sim\calN(0,0.2)$. An illustration of the sampled reference data set is illustrated in \cref{fig:hcillu}.
		\begin{figure}
		\centering
		\includegraphics[width=0.4\textwidth]{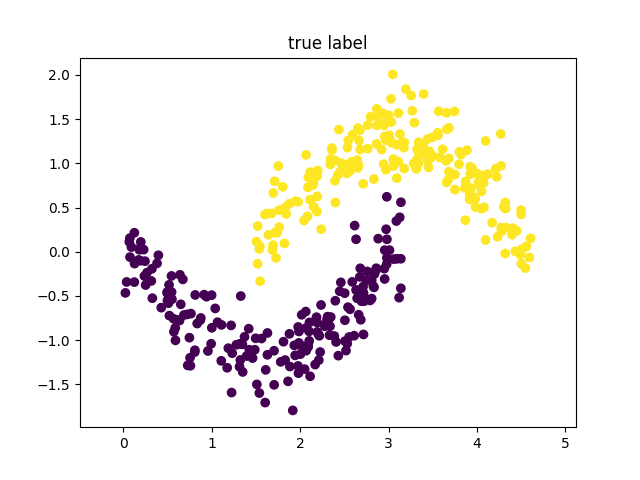}
		\includegraphics[width=0.4\textwidth]{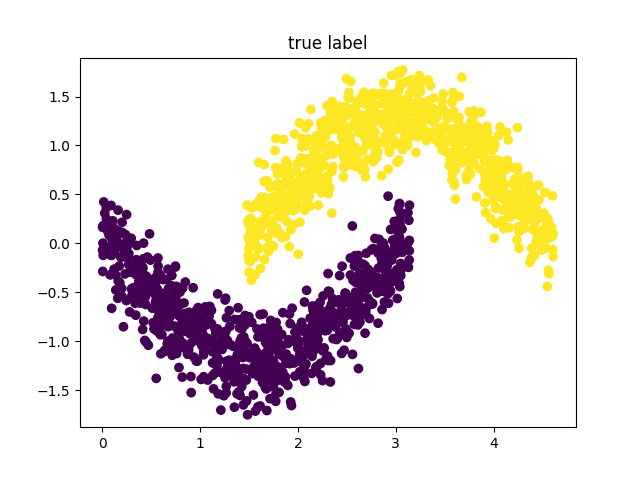}
		\caption{Sampled reference data sets of the point cloud in circle data set for $n=400$ (left) and $n=1600$ (right).}
		\label{fig:hcillu}
	\end{figure}
	The perturbed data is generated in analogy to the point in circle example in \cref{sec:pcic}. The runtime is also similar to the point sets in \cref{sec:pcic}, with $10^4$ samples taken for the Monte Carlo sampling.
	
	The obtained expected misclustering rate is illustrated in \cref{fig:hcmisc}, whereas the corresponding coverage functions, Vorob'ev expectation, ODF-expectation, and spectral expectation are depicted in \cref{fig:hcexp400} for $n=400$ and in \cref{fig:hcexp1600} for $n=1600$.
	\begin{figure}
		\centering
		\includegraphics[width=0.4\textwidth]{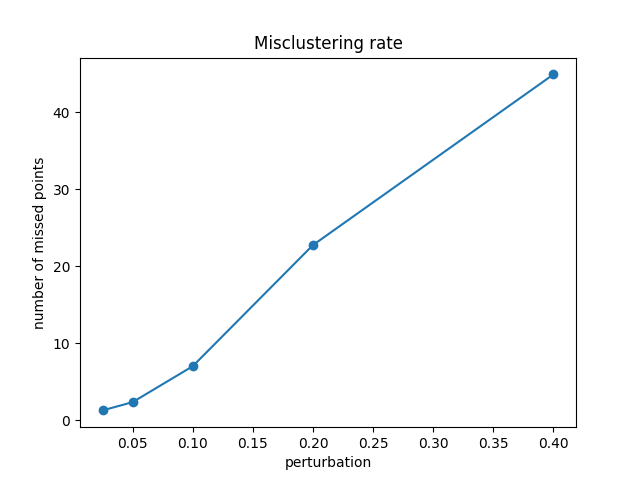}
		\includegraphics[width=0.4\textwidth]{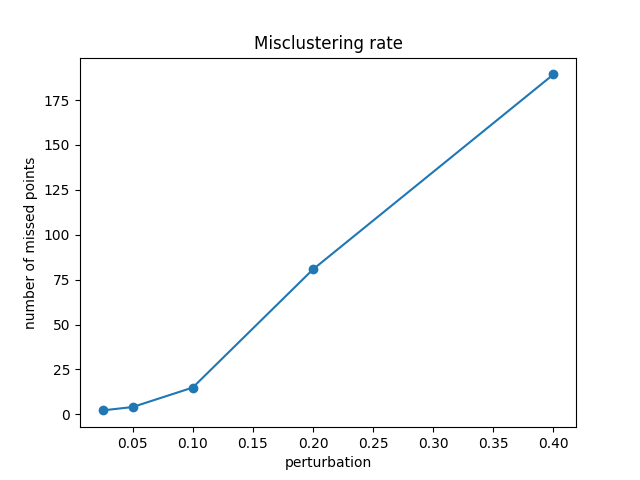}
		\caption{Expected misclustering rate for the entangled half circles data set for $n=400$ (left) and $n=1600$ (right).}
		\label{fig:hcmisc}
	\end{figure}
	\begin{figure}
		\centering
		\includegraphics[width=\textwidth]{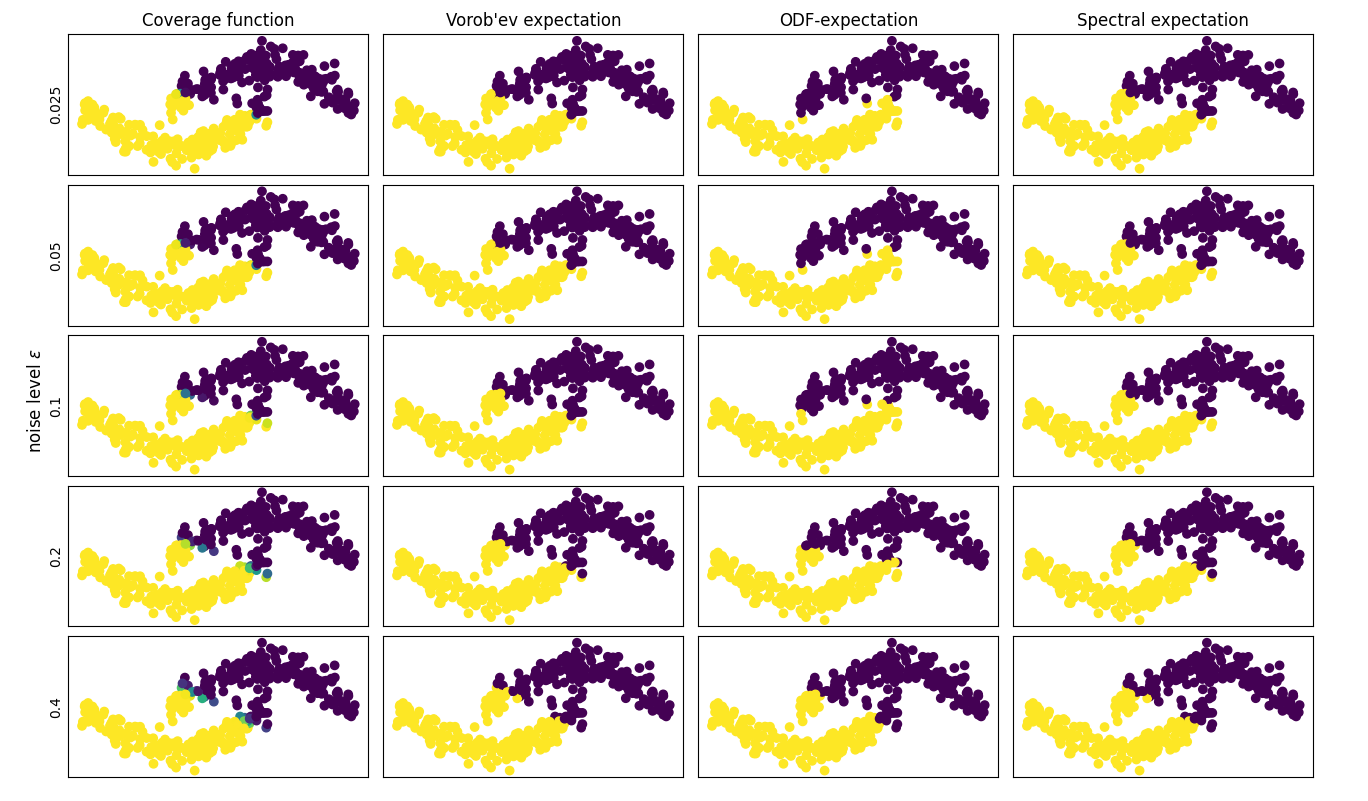}
		\caption{Approximated coverage function, Vorob'ev expectation, ODF-expectation, and spectral expectation for different noise levels for the entangled half circles set with $n=400$.}
		\label{fig:hcexp400}
	\end{figure}
	\begin{figure}
		\centering
		\includegraphics[width=\textwidth]{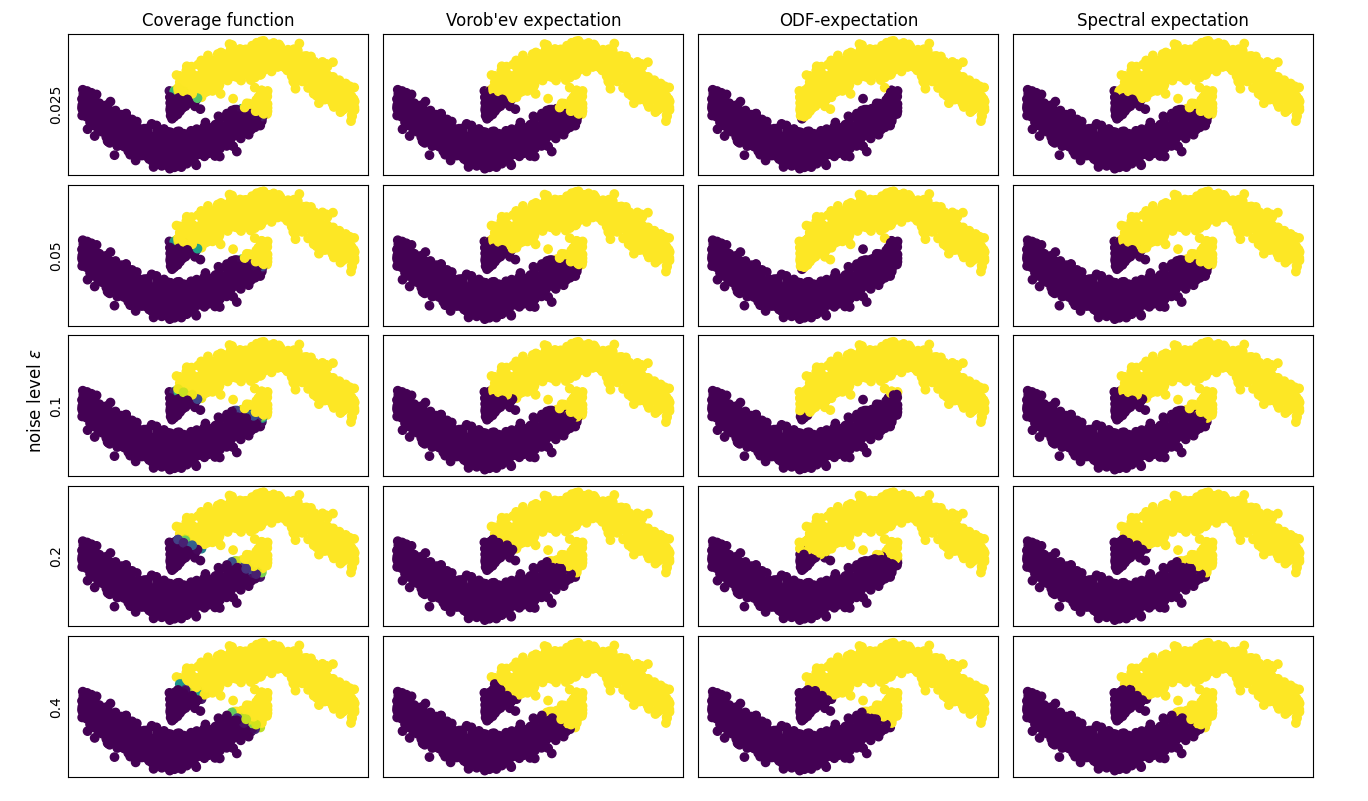}
		\caption{Approximated coverage function, Vorob'ev expectation, ODF-expectation, and spectral expectation for different noise levels for the entangled half circles data set with $n=1600$.}
		\label{fig:hcexp1600}
	\end{figure}
	The illustrations generally confirm the impressions from the first example from \cref{sec:pcic}.

	\subsection{Wisconsin breast cancer data set}
	For our final example we consider the Wisconsin breast cancer data set \cite{SWM1993} which is freely available within the library \verb+scikit-learn+ \cite{PVG+2011}. It contains 569 data points in a 30 dimensional space which are classified into malignant and benign. We normalize it in each of the 30 dimensions and perturb it with additive noise as in the previous two examples.
	Considering the set has 569 points, the runtime is comparable to the previous two examples.
	
	For visualization purposes, we project the data set into the coordinate system of the leading two directions of the data set obtained by PCA. An illustration of the original clustering as well as the misclustering rate under data perturbation is given in \cref{fig:wisconsin}.
	\begin{figure}
		\centering
		\includegraphics[width=0.4\textwidth]{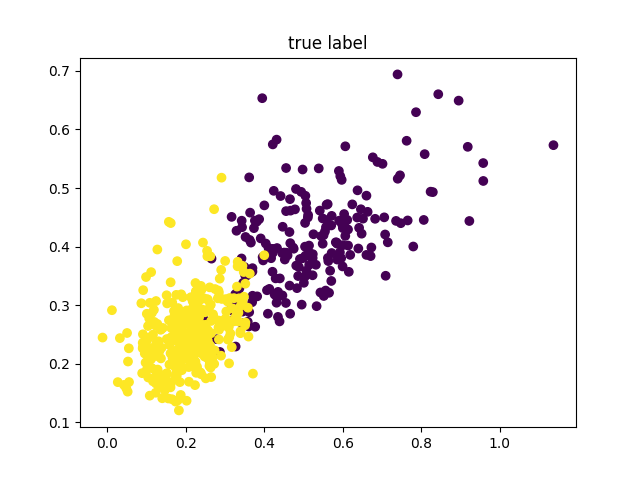}
		\includegraphics[width=0.4\textwidth]{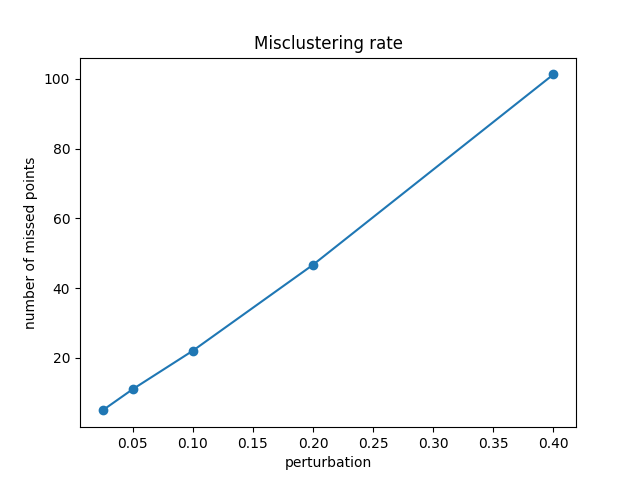}
		\caption{Normalized 2d projection of the Wisconsin breast cancer data set as well as the obtained misclustering noise under various additive noise levels.}
		\label{fig:wisconsin}
	\end{figure}
	The coverage functions, Vorob'ev expectation, ODF-expectation, and spectral expectation are depicted in \cref{fig:wisconsinexp}.
	
	\begin{figure}[p]
	\centering
	\begin{subfigure}[b]{\textwidth}
		\centering
		\includegraphics[width=\linewidth]{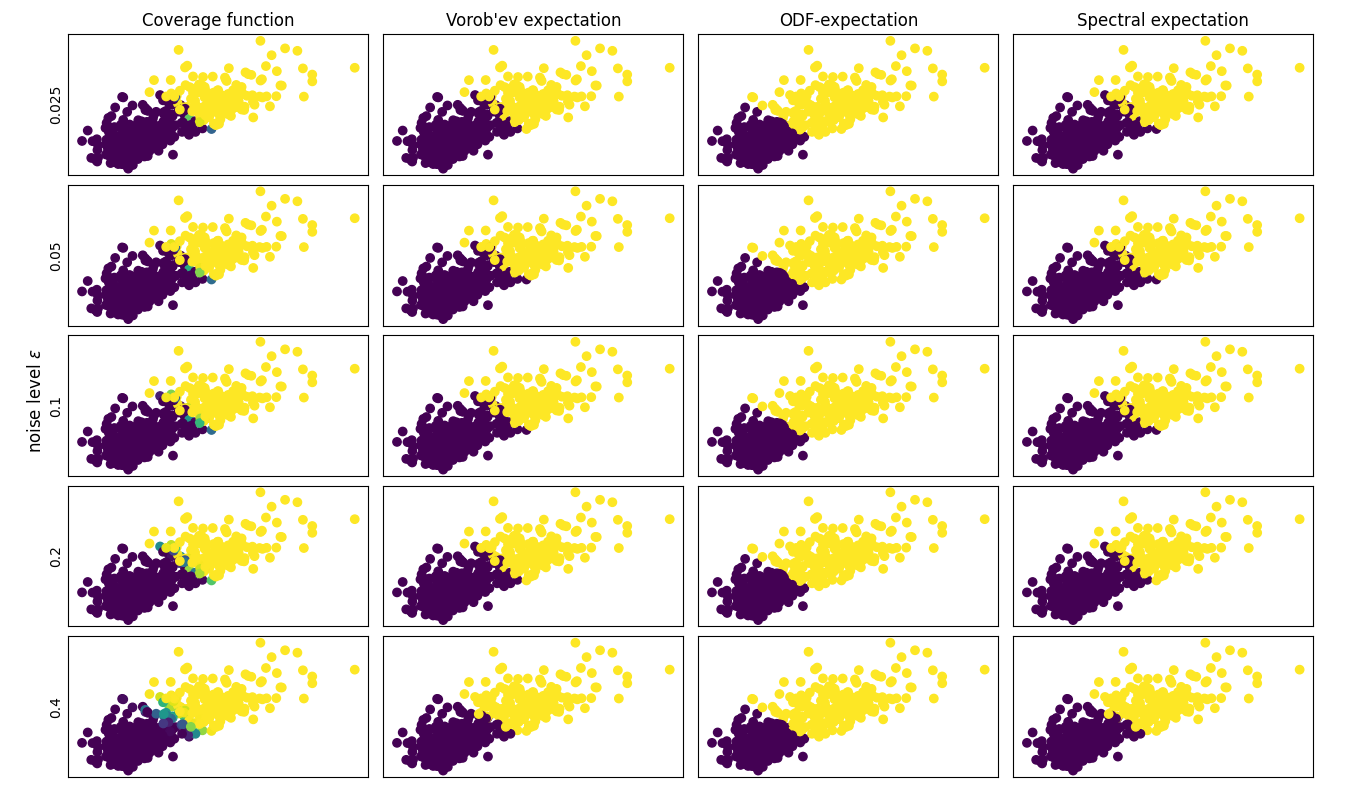}
		\caption{Obtained quantities as obtained from the algorithm.}
	\end{subfigure}
	\vspace{0.5em} %
	\begin{subfigure}[b]{\textwidth}
		\centering
		\includegraphics[width=\linewidth]{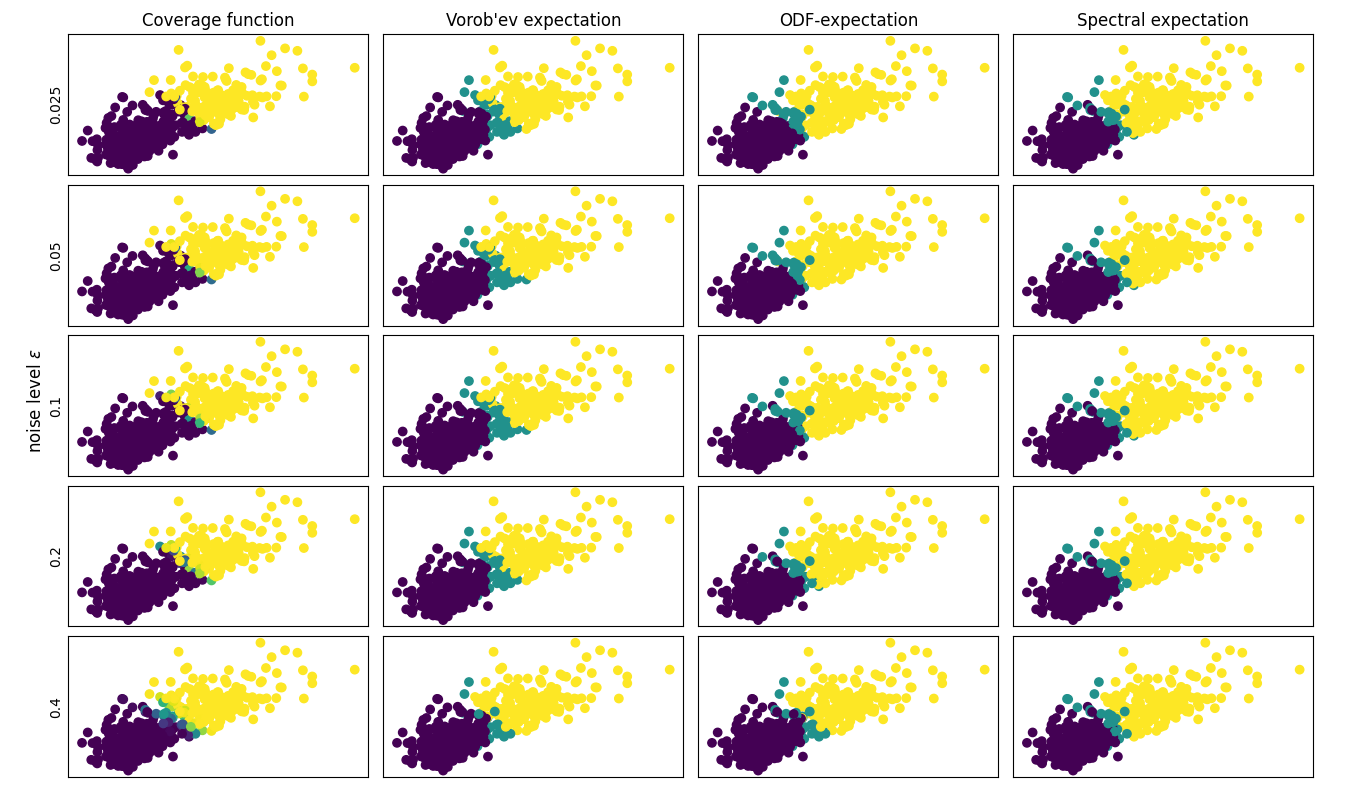}
		\caption{Obtained quantities with deviations from the original clustering, cf.\ \cref{fig:wisconsin}, shown in green.}
		\label{fig:brustkrebs_neu}
	\end{subfigure}
	\caption{Normalized 2d projections of the approximated coverage function, Vorob'ev expectation, ODF-expectation, and spectral expectation for different noise levels for the Wisconsin breast cancer data set.}
	\label{fig:wisconsinexp}
	\end{figure}

	\FloatBarrier %
	\section{Conclusion}\label{sec:concl}
	We propose an approach to compute expectations of spectral clusterings arising from corrupted data. Adopting an Eulerian approach, the corrupted data can arise from the ground truth due to perturbations in measurements, missing measurements, and even additional data points. Considering the arising clusters as random sets, we use random set theory to approximate the coverage function, the expected misclustering rate, the Vorob'ev expectation, the ODF-expectation by Monte Carlo approaches. We propose the spectral expectation as a natural expectation for clusterings in the context of spectral clusterings and Monte Carlo estimation. We provide a consistency analysis towards the infinite data point and infinite Monte Carlo sample for the continuous analog of spectral clustering, the coverage function, the ODF-expectation, and the spectral expectation in the case of a Gaussian similarity function. The analysis shows that it is irrelevant in which order the limits are taken and could be extended to other similarity functions as long as a corresponding version of \cref{cor:efuniformconv} can be shown. For a consistency analysis of the Vorob'ev expectation, a promising strategy could be to adopt the perspective of \cite{CGBM2013}. The numerical experiments show that for the toy data sets from \cref{sec:pcic} and \cref{sec:ehc} the ODF-expectation seems to be more stable to larger perturbation in the data, whereas for the breast cancer data set it seems to be Vorob'ev expectation. A detailed investigation thereof is left as future work. Likewise, using tools from high-dimensional approximation and recent progress in the uncertainty quantification of eigenvalue problems is a promising topic to reduce the computational burden of Monte Carlo sampling.
	
	\section{Acknowledgements}
	The authors would like to express their gratitude to Jochen Garcke for raising their interest into spectral clustering. The work of the authors was partially supported by the Deutsche Forschungsgemeinschaft (DFG, German Research Foundation), project 501419255. The authors also received support from the DFG under Germany’s Excellence Strategy, project 390685813.
	
	\bibliographystyle{plain}
	\bibliography{references}

\end{document}